\documentclass[twoside]{article}

\usepackage[accepted]{arxiv_style}

\usepackage{array,multirow,graphicx}

\usepackage{wrapfig}

\def\halfwidth{2.750535in}

\makeatletter
\newcommand\notsotiny{\@setfontsize\notsotiny\@vipt\@viipt}
\makeatother
\usepackage{gav_style}
\usepackage{thmtools, thm-restate}

\usepackage[utf8]{inputenc} 
\usepackage[T1]{fontenc}    
\usepackage[hidelinks]{hyperref}       
\usepackage{url}            
\usepackage{booktabs}       
\usepackage{amsfonts}       
\usepackage{nicefrac}       
\usepackage{microtype}      
\usepackage{xcolor}         
\usepackage{subcaption}
\usepackage[export]{adjustbox}

\begin{document}

\twocolumn[

\aistatstitle{Functional Flow Matching}

\aistatsauthor{Gavin Kerrigan \And Giosue Migliorini \And Padhraic Smyth}

\aistatsaddress{ Department of Computer Science\\ University of California, Irvine \\ \texttt{gavin.k@uci.edu} \And   Department of Statistics \\ University of California, Irvine \\ \texttt{gmiglior@uci.edu} \And  Department of Computer Science\\ University of California, Irvine \\ \texttt{smyth@ics.uci.edu}  } ]

\begin{abstract}
    We propose Functional Flow Matching (FFM), a function-space generative model that generalizes the recently-introduced Flow Matching model to operate in infinite-dimensional spaces. Our approach works by first defining a path of probability measures that interpolates between a fixed Gaussian measure and the data distribution, followed by learning a vector field on the underlying space of functions that generates this path of measures. Our method does not rely on likelihoods or simulations, making it well-suited to the function space setting. We provide both a theoretical framework for building such models and an empirical evaluation of our techniques. We demonstrate through experiments on several real-world benchmarks that our proposed FFM method outperforms several recently proposed function-space generative models.
\end{abstract}

\section{Introduction}
    \label{sect:intro}

    Generative models have seen a meteoric rise in capabilities on various domains, such as images \citep{dhariwal2021diffusion, kang2023scaling}, video \citep{saharia2022photorealistic, ho2022imagen}, and audio \citep{kong2020diffwave, goel2022s}. Despite these recent successes, many methods implicitly assume that the data distribution of interest is supported on a finite-dimensional space. However, there are multiple important applications that involve data that is inherently infinite-dimensional. For instance, draws from a time series, solutions to partial differential equations, and audio signals are naturally represented as \emph{functions}.
    
    For functional data, the typical generative modeling approach is to operate directly on a discretization of the data \citep{tashiro2021csdi, rasul2021autoregressive, yan2021scoregrad}. However, these approaches are often tied to a chosen discretization and are often ill-posed in the limit of zero discretization error. To overcome these limitations, recent methods have proposed building generative models directly in function spaces. For instance, \citet{kerrigan2022diffusion} and \citet{lim2023score} propose function-space diffusion models, and  \citet{rahman2022generative} propose a function-space GAN. 
    
    We add to this growing literature on function-space generative models. In particular, we propose Functional Flow Matching (FFM), a continuous-time normalizing flow model for functional data. Given that Euclidean normalizing flow methods \citep{papamakarios2021normalizing, kobyzev2020normalizing} are posed in terms of densities, which generally do not exist in infinite-dimensional spaces, a key challenge of performing this generalization is to pose a purely measure-theoretic model. In particular, our model is a generalization of the recently proposed Flow Matching model of \citet{lipman2022flow}. 
    
    Our proposed FFM model first constructs a path of conditional Gaussian measures, approximately interpolating between a fixed reference Gaussian measure and a given function. A path of measures interpolating between said reference measure and the data distribution is then obtained by marginalizing these conditional paths over the data distribution. We then learn a vector field on our space of functions which approximately generates this path of measures, allowing us to generate samples from our data distribution by solving a differential equation parametrized by this vector field. See Figure \ref{fig:leading_figure} for an illustration of our approach.

    Our approach allows for simulation-free training, in the sense that no samples are drawn from the model during training. Moreover, our training objective is a regression objective, allowing us to avoid difficulties with regards to maximum likelihood training in a functional setting. We empirically verify our framework on several time series datasets and a two-dimensional Navier-Stokes dataset, demonstrating that the FFM model significantly outperforms several recently proposed function-space generative models across a variety of metrics and domains. 
    
    \begin{figure*}[!t]
        \centering
        \includegraphics{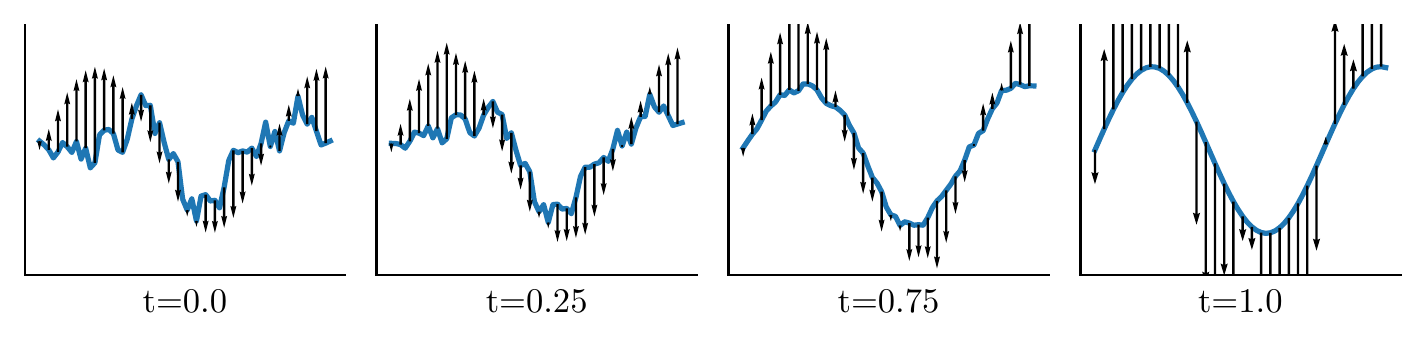}
        \caption{An illustration of our FFM method. The vector field $v_t(f) \in \FF$ (in black) transforms a noise sample $g \sim \mu_0 = \NN(0, C_0)$ drawn from a Gaussian process with a Mat\'ern kernel (at $t = 0$) to the function $f(x) = \sin(x)$ (at $t=1$) via solving a function-space ODE. By sampling many such $g \sim \mu_0$, we define a conditional path of measures $\mu_t^f$ approximately interpolating between $\NN(0, C_0)$ and the function $f$, which we marginalize over samples $f \sim \nu$ from the data distribution in order to obtain a path of measures approximately interpolating between $\mu_0$ and $\nu$.}
        \label{fig:leading_figure}
    \end{figure*}

    \section{Related Work}
    
    \paragraph{Flow Matching and Normalizing Flows} We generalize the Flow Matching model of \citet{lipman2022flow}, which is a novel approach to simulation-free continuous-time normalizing flows \citep{chen2018neural, papamakarios2021normalizing, kobyzev2020normalizing} that has demonstrated impressive capabilities on several image generation tasks. However, Flow Matching and other recently proposed simulation-free continuous normalizing flows have only been explored for data distributions supported on finite-dimensional spaces, such as Euclidean spaces \citep{lipman2022flow, albergo2022building, liu2022flow, neklyudov2022action} and Riemannian manifolds \citep{chen2023riemannian, ben2022matching}. In contrast, we propose Functional Flow Matching, a continuous-time normalizing flow for infinite-dimensional data. To the best of our knowledge, this is the first normalizing flow model posed in infinite-dimensional spaces.
    
    \paragraph{Function Space Generative Models} Recently, a number of authors have proposed function-space generalizations of various deep generative models. Close in spirit to our work are those which generalize diffusion models \citep{ho2020denoising, song2020score, song2019generative} to the infinite-dimensional setting. In particular, \citet{kerrigan2022diffusion} and \citet{lim2023score} propose function-space generalizations of discrete-time diffusion models, whereas \citet{pidstrigach2023infinite}, \citet{franzese2023continuous} and \citet{hagemann2023multilevel} propose function-space generalizations of continuous-time diffusion models. Beyond diffusion models, function-space GANs \citep{rahman2022generative} and energy-based models \citep{lim2023energy} have also been proposed. Our work adds to this growing literature on function-space generative models by proposing the first function-space normalizing flow model. 
        
    \paragraph{Discrete Functional Generative Models} While there has been growing interest in developing generative models directly in infinite-dimensional spaces, there has also been work proposing generative models for functional data that operate directly on a discretization of the underlying space, for example, diffusion models for time series \citep{tashiro2021csdi, rasul2021autoregressive, yan2021scoregrad} (see \citet{lin2023diffusion} for a recent survey on these methods). Other models, such as normalizing flows \citep{rasul2020multivariate}, latent variable models \citep{zhou2022deep, rubanova2019latent, yildiz2019ode2vae}, and GANs \citep{yoon2019time, kidger2021neural} have also been explored. However, these methods all operate directly on the discrete observations of a given time series. This has several drawbacks: for instance, it is difficult to transfer a model trained on one discretization to another, and often these models are ill-posed in the functional limit (i.e. as the discretization size goes to zero). In contrast, our work begins from a function-space point of view, where we only  discretize in order to perform computations.
    
\section{Notation and Background}
\label{sect:notation_background}

    We begin by introducing some notation and background which we will later use to construct our model. Section \ref{sect:prelims} introduces notions related to flows on function spaces, and Section \ref{sect:weak_continuity} introduces the weak continuity PDE \citep{stepanov2017three} which plays a key role in our constructions.
    
\subsection{Preliminaries}
\label{sect:prelims}

    Let $\XX \subset \R^d$ and consider a real separable Hilbert space $\FF$ of functions $f: \XX \to \R$ equipped with the Borel $\sigma$-algebra $\BB(\FF)$. We consider the setting where there is a probability measure $\nu$ on $\FF$ from which we have samples, i.e. random functions drawn from the data distribution $\nu$. Our goal is to build a generative model which allows us to sample from $\nu$. Importantly, any such generative model should be discretization-invariant, in the sense that the model should be able to generate functions which may be observed on any finite but arbitrary discretization of $\XX$.

    In this work, we consider paths of probability measures $(\mu_t)_{t \in [0, 1]}$, where $\mu_t \in \PP(\FF)$ is a probability measure on $\FF$ at every time $t \in [0, 1]$. In particular, we will construct a path of measures which approximately interpolates between a fixed reference measure $\mu_0$ at time $t=0$ and the data distribution at time $t=1$, so that $\mu_1 \approx \nu$. This interpolation is approximate in the sense that $\mu_1$ will be a smoothed version of the data distribution, obtained from $\nu$ via convolution with a Gaussian measure having small variance \citep[Appendix~A]{bogachev1998gaussian}. 
    
    We consider paths of probability measures which are generated locally, in the sense that there is some underlying time-dependent vector field on $\FF$ such that the path of measures $(\mu_t)_{t \in [0, 1]}$ is obtained by flowing samples $g \sim \mu_0$ along said vector field. More formally, a \emph{(time-dependent) vector field} on $\FF$ is a mapping $v: [0, 1] \times \FF \to \FF$.
    
    The \emph{flow} associated to a vector field $(v_t)_{t \in [0, 1]}$ is the mapping $\phi: [0, 1] \times \FF \to \FF$ specified by the initial value problem
    \begin{equation} \label{eqn:flow_ode}
        \partial_t \phi_t(g) = v_t(\phi_t(g)) \qquad \phi_0(g) = g
    \end{equation}

    As written, Equation \eqref{eqn:flow_ode} represents an ordinary differential equation (ODE) on the abstract and potentially infinite-dimensional space $\FF$. Such ODEs are often dubbed \emph{abstract differential equations} \citep{ORegan1997, zaidman1999functional}. We assume that all vector fields in this work are sufficiently regular such that a solution to Equation \eqref{eqn:flow_ode} is guaranteed to exist for all $t \in [0, 1]$ and $\nu$-a.e. initial condition $g$.
    
    Given any initial probability measure $\mu_0 \in \PP(\FF)$, we may consider the path of probability measures generated by the flow $\phi$. That is, for any $t \in [0, 1]$ we define the measure $\mu_t$ via the pushforward of $\mu_0$ along $\phi_t$, i.e. $\mu_t = [\phi_t]_{\#}\mu_0$, where $\phi_t$ is assumed to be measurable for all $t \in [0, 1]$.  

\subsection{Weak Continuity PDE}
\label{sect:weak_continuity}

    Previously, we noted how one may obtain a path of probability measures from an initial probability measure $\mu_0 \in \PP(\FF)$ by considering the pushforward of $\mu_0$ along the flow of a given vector field $(v_t)_{t \in [0, 1]}$.

    Conversely, we say that the vector field $(v_t)_{t \in [0, 1]}$ \emph{generates} the path of measures $(\mu_t)_{t \in [0, 1]}$ if the path $(\mu_t)_{t \in [0, 1]}$  is obtained via the pushforward of $\mu_0$ along the flow associated with $(v_t)_{t \in [0, 1]}$. Directly verifying whether a vector field generates a given path of measures (by verifying the pushforward relationship) is typically infeasible. Instead, we can check that the two satisfy the \emph{continuity equation}

    \begin{equation} \label{eqn:continuity_pde}
        \partial_t \mu_t + \text{div}(v_t \mu_t) = 0 \qquad \text{on } \; \FF \times [0, 1].
    \end{equation}

    We interpret this PDE in the weak sense \citep[Ch.~8]{ambrosio2005gradient}, by which we mean that the pair $(v_t)_{t \in [0, 1]}$ and $(\mu_t)_{t \in [0, 1]}$ satisfy Equation \eqref{eqn:continuity_pde} if
    \begin{equation} \label{eqn:weak_continuity_pde}    
         \int_0^1 \int_\FF \left( \partial_t \varphi(g, t) + \langle v_t(g), \nabla_g \varphi(g, t) \right) \d\mu_t (g) \d t = 0
    \end{equation}

    for all $\varphi: \FF \times [0, 1] \to \R$ in some appropriate space of test functions. 
    
    We refer to \citet[Theorem~3.4]{stepanov2017three} for a rigorous discussion of this result in general metric spaces. Such results are often referred to as \emph{superposition principles}. In the Euclidean setting, if one assumes that all measures admit a density with respect to some common dominating measure, it suffices to check the continuity equation directly, in which $\mu_t$ is replaced by a density $p_t$ \citep{ambrosio2005gradient, villani2009optimal}.
    
    Throughout this work, we assume all paths of measures and vector fields are sufficiently regular such that the superposition principle applies, i.e. it suffices to check the continuity equation to conclude whether a given path of measures is generated by a given vector field. In Theorem \ref{theorem:gluing_formula}, we use the weak form of the continuity equation in order to construct a marginal vector field from conditional vector fields, such that this marginal vector field is guaranteed to generate our desired interpolating path of measures.

\section{Function Space Flow Matching}
\label{sect:functional_flow_matching}
    
    Building on the notions in Section \ref{sect:notation_background}, we now introduce our Functional Flow Matching model (FFM). The Flow Matching model is a recently proposed continuous-time normalizing flow method developed for finite-dimensional spaces \citep{lipman2022flow, chen2023riemannian}. Our FFM approach builds on this earlier line of work to develop a non-trivial extension of these methods to infinite-dimensional spaces. Proofs for all of our claims are available in Appendix \ref{appendix:proofs}.
    
    The main technical challenge of generalizing the existing techniques to infinite-dimensional spaces is that existing methods rely heavily on the notion of a probability density function, either with respect to the Lebesgue measure in the case of a Euclidean space or with respect to the canonical volume measure on a Riemannian manifold. In infinite-dimensional (Banach) spaces, there does not exist an analogue of the Lebesgue measure -- that is, any nonzero translation invariant Borel measure must assign infinite measure to any open set \citep{eldredge2016analysis}. 

    \begin{figure*}[!t]
    \centering
        \begin{subfigure}[C]{0.33\textwidth}
            \centering
            \includegraphics[width=\textwidth]{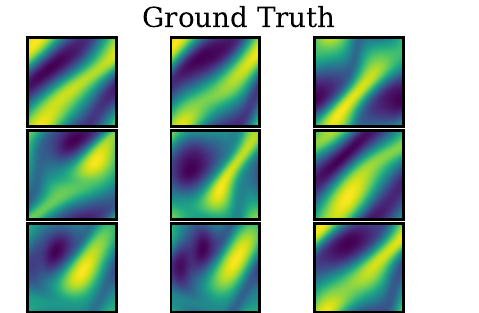}
        \end{subfigure}
        \hfill
        \begin{subfigure}[C]{0.66\textwidth}  
            \centering 
            \includegraphics[width=\textwidth]{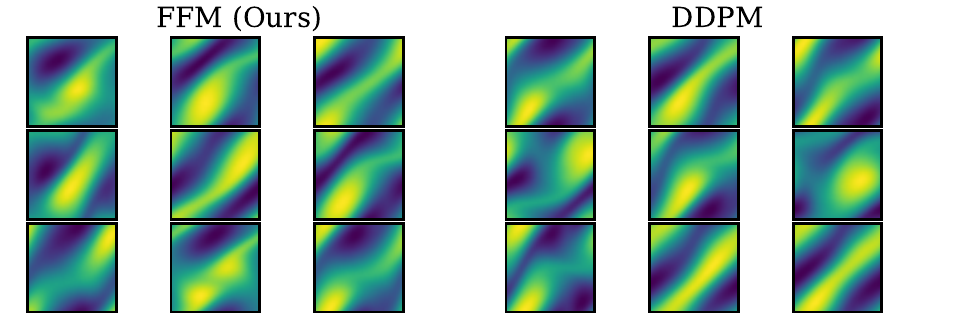}  
        \end{subfigure} 
        \begin{subfigure}[C]{0.33\textwidth}
            \centering
            \caption{MSEs between the density and spectra of the real and generated samples.}
            \label{tab:ns_results}
            \resizebox{!}{0.2\textwidth}{%
            \begin{tabular}{l|cc}
            \toprule
                               & Density & Spectrum \\
            \midrule
                 FFM-OT (Ours) & \textbf{3.7e-5}  & \textbf{9.3e1}    \\
                 DDPM          & 9.9e-5           & 5.0e2    \\
                 DDO           & 2.9e-2           & 1.6e5    \\
                 GANO          & 2.5e-3           & 3.2e4    \\
             \bottomrule
            \end{tabular}
            }
            \label{tab:my_label}
        \end{subfigure}
        \hfill
        \begin{subfigure}[C]{0.66\textwidth}   
            \centering 
            \includegraphics[width=\textwidth]{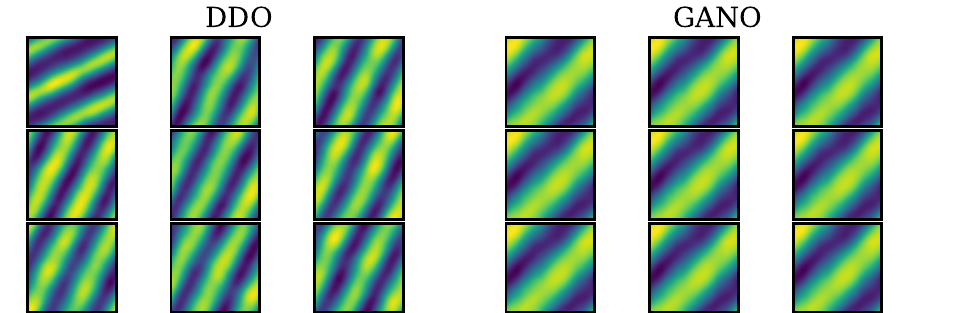}   
            \label{fig:mean and std of net44}
        \end{subfigure}
        \caption[]
        {\small Samples from the Navier-Stokes dataset (``ground truth'') and samples from the various models considered in this work. Our FFM-OT model and DDPM qualitatively match the ground truth samples, whereas DDO and GANO suffer from mode collapse. Table \ref{tab:ns_results} compares the density and spectra between 1000 real and generated samples, showing that our proposed method outperforms the others by a large margin on pointwise metrics. Note that we do not study the FFM-VP parametrization on this dataset due to computational costs.}
        \label{fig:ns_samples}
    \end{figure*}
    
    As such, our FFM model is necessarily posed in measure-theoretic terms. Our derivations shed light on strict requirements needed to obtain a well-posed model. For instance, we require an absolute continuity assumption between the conditional and marginal measures defined in Section \ref{sect:construction_paths}. In Euclidean spaces, such assumptions are easy to satisfy, but are non-trivial in infinite-dimensional spaces, even for the simple setting of Gaussian measures (see Section \ref{sect:ac_of_mixtures}). Moreover, our derivations demonstrate that naively applying a white-noise Gaussian measure (as is done in the Euclidean setting) leads to an ill-posed model in function space.
    
    \subsection{Constructing a Path of Measures}
    \label{sect:construction_paths}
    
    Suppose we associate to every $f \in \FF$ a path of measures $(\mu_t^f)_{t \in [0, 1]}$ such that $\mu_0^f = \mu_0$ is some fixed reference measure and $\mu_1^f$ is concentrated around $f$. For instance, $\mu_1^f$ could be a Gaussian measure with mean $f$ and a covariance having small operator norm. We then marginalize over all such measures, where we mix over the data distribution $\nu$. That is, we define a new probability measure $\mu_t \in \PP(\FF)$ for $t \in [0, 1]$ via
    
    \begin{equation} \label{eqn:marginal_measure}
        \mu_t(A) = \int \mu_t^f(A) \d \nu(f) \qquad \forall A \in \BB(\FF).
    \end{equation}
    
    Due to our conditions on $\mu_t^f$, we then have that $\mu_0 = \mu_0$ and $\mu_1 \approx \nu$ is approximately the data distribution. Suppose further that each conditional path of measures $\mu_t^f$ is generated by some known vector field $v_t^f$. In the following theorem, we claim that we may construct a vector field $v_t$ which generates the \emph{marginal} path of measures $\mu_t$ from the conditional vector fields $v_t^f$. 
    
    \begin{restatable}{theorem}{gluing} \label{theorem:gluing_formula}
        Assume that $\int_0^1 \int_{\FF\times\FF} || v_t^f(g)|| \d \mu_t^f(g) \d\nu(f) \d t < \infty$. If $\mu_t^f \ll \mu_t$ for $\nu$-a.e. $f$ and almost every $t \in [0, 1]$, then the vector field
    
        \begin{equation} \label{eqn:marginal_vector_field}
        v_t(g) = \int_\FF v_t^{f}(g) \frac{\d \mu_t^{f}}{\d \mu_t}(g) \d \nu(f)
        \end{equation}
    
        generates the marginal path of measures measure $(\mu_t)_{t\in [0,1]}$ specified by Equation \eqref{eqn:marginal_measure}. That is, $(v_t)_{t \in [0,1]}$ and $(\mu_t)_{t \in [0, 1]}$ solve the continuity equation \eqref{eqn:continuity_pde}. Here, $\d \mu_t^f / \d \mu_t$ is the Radon-Nikodym derivative of the conditional measure with respect to the marginal.
    \end{restatable}
    
    If this vector field $v_t$ were known, we could generate samples by solving the corresponding flow ODE (Equation \eqref{eqn:flow_ode}) with initial condition $f \sim \mu_0$ drawn from our fixed reference measure. However, the vector field specified by Equation \eqref{eqn:marginal_vector_field} is intractable. Thus, we will learn a model to approximate this unknown vector field. Note that our model will necessarily be a mapping between infinite dimensional spaces. We discuss how to parametrize such a model in Section \ref{sect:experiments}.
    
    The main technical assumption in Theorem \ref{theorem:gluing_formula} is that the conditional distributions $\mu_t^f$ are $\nu$-almost surely absolutely continuous with respect to the marginal distribution $\mu_t$. Although this assumption is not generally true even in the Euclidean setting, we prove in Theorem \eqref{theorem:mutually_ac_mixture} that this assumption holds under an additional equivalency condition on the conditional measures. In Section \ref{sect:ac_of_mixtures}, we discuss how this equivalency assumption may be satisfied under a Gaussian parametrization. 
    
    \begin{restatable}{theorem}{mutuallyac}
    \label{theorem:mutually_ac_mixture}
        Consider a probability measure $\nu$ on $\FF$ and a collection of measures $\mu_t^f$ parametrized by $f \in \FF$. Suppose that the collection of parametrized measures are $\nu$-a.e. mutually absolutely continuous. Define the marginal measure $\mu_t$ via Equation \eqref{eqn:marginal_measure}.Then, $\mu_t^f \ll \mu_t$ for $\nu$-a.e. $f$. 
    \end{restatable}

\subsection{Special Case: Gaussian Measures}
\label{sect:gaussian_setting}

    In this section, we specialize to the setting where the reference measure $\mu_0$ and conditional measures $\mu_t^f$ are chosen to be Gaussian measures \citep{bogachev1998gaussian}. We make this ansatz for several reasons. Foremost, our marginal vector field (Equation \eqref{eqn:marginal_vector_field}) requires an absolute continuity assumption. In infinite-dimensional (separable) Banach spaces, the absolute continuity of Gaussian measures is well-understood, e.g. via the Cameron-Martin theorem and the Feldman-H\'ajek theorem \citep{da2014stochastic, bogachev1998gaussian}. Moreover, we are able to parametrize our Gaussian measures via Gaussian processes \citep{rasmussen2006gaussian, wild2022generalized} for which a number of flexible choices of kernels have been explored in the machine learning literature.
    
    More formally, for any $f\in \FF$ we define a conditional path of probability measures $(\mu_t)_{t \in [0, 1]}$ to be a Gaussian measure $\mu_t^f = \NN(m_t^f, C_t^f)$ with mean $m_t^f \in \FF$ and covariance operator $C_t^f: \FF \to \FF$. Note that the $C_t^f$ are necessarily symmetric, non-negative and trace-class \citep[Ch.~2]{da2014stochastic}. In particular, this rules out multiples of the identity operator (corresponding to white noise) as a valid choice for $C_t^f$, as these operators are not compact and hence not trace-class.
    
    In practice, we parametrize $t \mapsto m_t^f$ by a Fr\'echet differentiable mapping and specify $C_t^f$ by a covariance operator $C_0$ and variance schedule $t \mapsto \sigma_t^f \in \mathbb{R}_{> 0 }$ such that $C_t^f = (\sigma_t^f)^2 C_0$. At time $t=0$, we choose to parametrize $\mu_0^f = \mu_0 = \NN(0, C_0)$ as a centered Gaussian measure independent of the function $f \in \FF$. The measure $\mu_0$ will serve as the reference measure in our generative model. In order to satisfy the desiderata of Section \ref{sect:construction_paths}, at time $t=1$ we will choose $m_1^f = f$ and $C_t^f$ to have small operator norm so that $\mu_1^f$ is a Gaussian measure concentrated around $f$. 
    
    In this case, we note that the conditional flow $\phi^f: [0, 1] \times \FF \to \FF$ defined via $\phi_t^f(g) = \sigma_t^f g + m_t^f$  will push $g \sim \NN(0, C_0)$ to the desired conditional measure $\mu_t^f$, i.e. $\mu_t^f = [\phi_t^f]_{\#} \NN(0, C_0)$. Using the flow ODE \eqref{eqn:flow_ode}, we see that a vector field generating this conditional path of measures is
    \begin{equation} \label{eqn:conditional_vector_field}
        v_t^f(g) = \frac{(\sigma_t^f)^\prime}{\sigma_t^f}(g - m_t^f) + \frac{\d}{\d t} m_t^f
    \end{equation}
    where $(\sigma_t^f)^\prime$ is the ordinary time derivative of the variance schedule and $\d / \d t (m_t^f)$ is the Fr\'echet derivative of the mapping $t \mapsto m_t^f$. The proof of this fact is a straightforward generalization of \citet[Theorem 3]{lipman2022flow}, which demonstrates the analogous relationship in finite-dimensional Euclidean spaces.

    In this work, we consider two concrete parameterizations. In the first parametrization (``OT''), the mean and variance are given as affine functions of $t$ and $f$:

    \begin{equation}
        \label{eqn:ot_param}
        m_t^f = tf \qquad \sigma_t^f = 1 - (1 - \sigma_{\min}) t.
    \end{equation}

    The ``OT'' path is named as such as it corresponds to an optimal transport map between Gaussians in the Euclidean setting \citep{lipman2022flow, mccann1997convexity}. 
    
    In the second parametrization (``VP''), we set
    \begin{equation}
        \label{eqn:vp_params}
        m_t^f = \alpha_{1-t} f \qquad \sigma_t^f = \sqrt{1 - \alpha_{1-t}^2}.
    \end{equation}

    This path is inspired inspired by probability paths defined via variance preserving diffusion models \citep{lipman2022flow, song2020score}. We additionally experimented with the ``variance exploding'' parametrization \citep{lipman2022flow, song2020score}, but found empirically that this was not suitable for our setting. See Appendix \ref{appendix:experiment_details} for details. Here, $\sigma_{\min} \in \R_{>0}$ and $\alpha_t \in \R_{>0}$ are hyperparameters of the model controlling the variance of the conditional measures.

\subsection{Absolute Continuity for Gaussians} 
\label{sect:ac_of_mixtures}

    In general, the absolute continuity assumption of Theorem \ref{theorem:gluing_formula} is difficult to satisfy in function spaces. In the Gaussian setting, we may reduce this assumption to assumptions regarding the parametrization of our Gaussian measures. By the Feldman-H\'ajek theorem \citep[Theorem~2.25]{da2014stochastic}, our conditional Gaussian measures $\mu_t^f$ will be mutually absolutely continuous if the difference in means lies in the Cameron-Martin space of $C_t$, i.e. $m_t^f - m_t^g \in C_t^{1/2}(\FF)$. 
    
    Thus, under suitable assumptions on the data distribution $\nu$ and an appropriate parametrization of the conditional means, our marginal vector fields (Equation \eqref{eqn:marginal_vector_field}) will be well-defined as a consequence of \mbox{Theorem \ref{theorem:mutually_ac_mixture}}. Suppose $C_t = \sigma_t^2 C_0$ is a scaled version of some fixed covariance operator $C_0$ with the assumption that $0 < \sigma_t^2 \leq M$ is positive and bounded above. By Lemma 6.15 of \citet{stuart2010inverse}, this choice guarantees us that the Cameron-Martin space is constant in time, i.e. $C_0^{1/2}(\FF) = C_t^{1/2}(\FF)$ for all $t \in [0, 1]$.

    Assume further that the data distribution is supported on the Cameron-Martin space of $C_0$, i.e. ${\nu(C_0^{1/2}(\FF)) = 1}$. In this case, given our covariance parametrization, our Gaussian measures will be mutually absolutely continuous if e.g. $m_t^f$ is an affine function of $f$.  We note that the parametrizations suggested in Section \ref{sect:gaussian_setting} are all affine, and so under the assumption that the data is supported on the Cameron-Martin space $C_0^{1/2}(\FF)$ our setup is well-defined.
    
    In practice, verifying whether the data distribution is supported on $C_0^{1/2}(\FF)$ is difficult. One option to guarantee this assumption is satisfied is to pre-process the data via some mapping ${T: \FF \to C_0^{1/2}(\FF) \subseteq \FF}$ whose image is contained in $C_0^{1/2}(\FF)$. We refer to Appendix C of \citet{lim2023score} for a further discussion of such mappings and related results. We note that in practice, we do not find it necessary to perform this pre-processing.

\subsection{Training the FFM Model}

    Ideally, we would like to perform functional regression on the marginal vector field defined via Equation \eqref{eqn:marginal_vector_field}, where we approximate $v_t(g)$ by a model $u_t(g \mid \theta)$ with parameters $\theta \in \R^p$. This could be achieved, for instance, by minimizing the loss
    \begin{equation}
        \LL(\theta) = \E_{t \sim \UU[0,1], g \sim \mu_t}\left[ \norm{v_t(g) - u_t(g \mid \theta)}^2 \right]
    \end{equation}

    where $\UU[0,1]$ denotes a uniform distribution over the interval $[0,1]$. Note here that our model is a mapping $u: \R^p \times [0, 1] \times \FF \to \FF$, i.e. our model is a parametrized, time-dependent operator on the function space $\FF$. However, such a loss is intractable to compute -- in fact, if we had access to $(v_t)_{t \in [0, 1]}$, there would be no need to learn a model. Consider instead the conditional loss, defined via
    \begin{align}\label{eqn:cond_loss}
        \JJ(\theta)&= \E_{t \sim \UU [0,1], f \sim \nu, g \sim \mu_t^{f} } \left[ \norm{v_t^f(g) - u_t(f \mid \theta)}^2 \right]
    \end{align}

    where, rather than regressing on the intractable $v_t$, we regress on the \emph{known} conditional vector fields $v_t^f$. In the following theorem, we claim that minimizing $\JJ(\theta)$ is equivalent to minimizing $\LL(\theta)$.

\begin{restatable}{theorem}{losstheorem}
    Assume that the true and model vector fields are square-integrable, i.e. $\int_0^1 \int_\FF \norm{v_t(g)}^2 \d \mu_t(g) \d t < \infty$ and $\int_0^1 \int_\FF \norm{u_t(g \mid \theta)}^2 \d \mu_t(g) \d t < \infty$. Then, 
    $\LL(\theta) = \JJ(\theta) + C$ where $C \in \R$ is a constant independent of $\theta$.
\end{restatable}

\begin{figure}[t]
    \centering
    \includegraphics[]{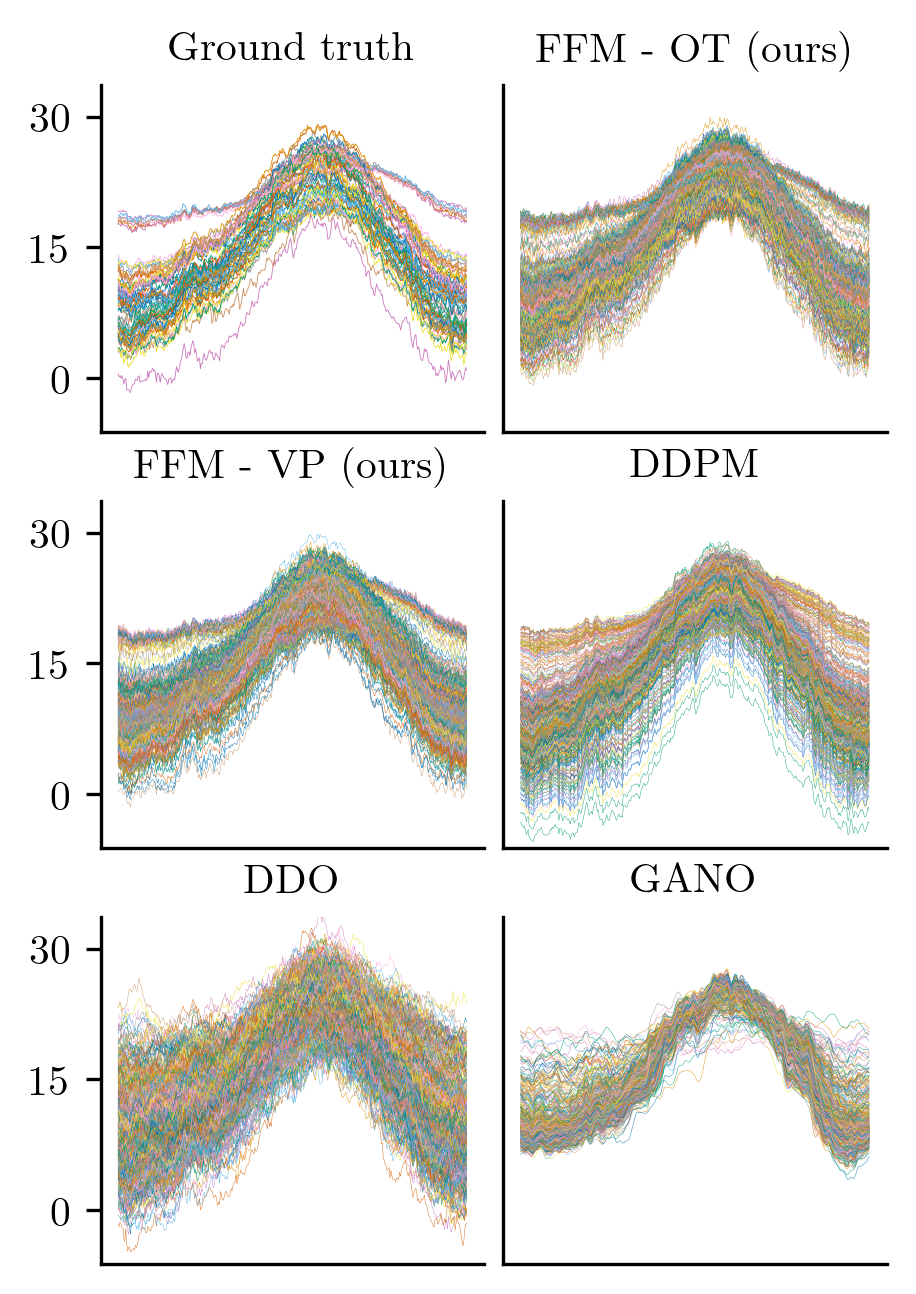}
    \caption{Unconditional generation of 500 samples on the AEMET dataset. Samples from our FFM model and DDPM appear visually to better match the characteristics of the real data relative to DDO and GANO.}
    \label{fig:unconditional_aemet}
\end{figure}

\section{Experiments}
\label{sect:experiments}
    
    We now investigate the empirical performance of FFM on several real-world datasets. In all settings, we assume we are working in the space ${\FF = L^2([0, 1])}$ and we parametrize $u_t(\blank \mid \theta)$ via a Fourier Neural Operator (FNO) \citep{li2020fourier}. Sampling is achieved by drawing a sample from the reference measure $g \sim \mu_0$ and numerically solving the flow ODE (Equation \eqref{eqn:flow_ode}) with initial condition $g$. In our implementation, we use the DOPRI solver \citep{DORMAND198019}. Details can be found in Appendix \ref{appendix:experiment_details}. Code for all of our experiments will be made available upon publication.

    \paragraph{Datasets} Our experiments in 1D include five datasets selected for their diverse correlation structures, exhibiting distinctive patterns that enable visual evaluation of generated samples. Plots of original and generated samples, as well as a detailed description of each dataset, can be found in Appendix \ref{appendix:datasets}. The first dataset (AEMET) consists of a set of 73 curves describing the mean daily temperature at various locations \citep{febrero2012statistical}. The second is a gene expression time series dataset \citep{orlando2008global}, and the remaining three consist of global economic time series on population, GDP per capita, and labor force size \citep{bolt2020maddison, inklaar2018rebasing, IMFlabor}. We also experiment with a dataset of solutions to the Navier-Stokes equation on a 2D torus \citep{li2022learning}.

\begin{table*}[t]
\centering
\caption{Average MSEs between true and generated samples for pointwise statistics on five 1D datasets, along with the standard deviation across ten random seeds. The average number of function evaluations (NFEs) for each sampling procedure in our implementation is also reported. Our FFM models obtain the best average performance across nearly all metrics, while simultaneously requiring fewer NFEs than the diffusion baselines.}
\resizebox{\textwidth}{!}{%
\notsotiny
\begin{tabular}{c|lcccccc}
\toprule
 && Mean & Variance & Skewness & Kurtosis & Autocorrelation & NFEs \\
\midrule
\parbox[t]{2mm}{\multirow{5}{*}{\rotatebox[origin=c]{90}{\textbf{AEMET}}}} & FFM-OT (ours) & \textbf{8.4e-2} (9.9e-2) & 1.7e+0 (1.1e+0) & 7.7e-2 (6.6e-2) & 3.3e-2 (3.7e-2) & \textbf{3.0e-6} (4.0e-6) & 668 \\
&FFM-VP (ours) & 1.3e-1 (1.4e-1) & \textbf{1.5e+0} (1.2e+0) & \textbf{5.e-2} (4.3e-2) & \textbf{1.7e-2} (1.6e-2) & 6.0e-6 (7.0e-6) & 488 \\
&DDPM          & 3.e-1 (3.0e-1) & 3.5e+0 (4.6e+0) & 2.2e-1 (2.2e-1) & 4.8e-2 (3.7e-2) & 1.2e-5 (9.e-6) & 1000\\ 
&DDO           & 2.4e-1 (1.4e-1) & 6.6e+0 (5.1e+0) & 2.1e-1 (4.1e-2) & 3.8e-2 (1.3e-2) & 6.7e-4 (1.3e-4) & 2000\\ 
&GANO          & 6.5e+1 (1.9e+2) & 3.7e+1 (4.0e+1) & 2.9e+0 (4.8e+0) & 3.3e-1 (4.0e-1) & 1.2e-3 (3.1e-3) & 1\\ 
\midrule
\parbox[t]{2mm}{\multirow{5}{*}{\rotatebox[origin=c]{90}{\textbf{Genes}}}} & FFM-OT (ours) & 6.7e-4 (4.5e-4) & 3.9e-3 (2.6e-4) & 2.4e-1 (4.7e-2) & 7.7e-2 (9.0e-3) & 2.5e-4 (1.7e-4) & 386
\\&FFM-VP (ours) & \textbf{4.2e-4} (3.8e-4) & \textbf{7.3e-4} (3.5e-4) & \textbf{1.9e-1} (6.1e-2) & \textbf{4.3e-2} (1.1e-2) & \textbf{1.3e-4} (1.0e-4) & 290\\
&DDPM          & 8.8e-4 (4.5e-4) & 1.9e-3 (4.2e-4) & 3.6e-1 (1.9e-1) & 6.3e-2 (1.1e-2) & 4.3e-4 (9.3e-5) & 1000\\ 
&DDO           & 4.2e-3 (1.5e-3) & 1.2e-3 (3.6e-4) & 3.0e-1 (5.7e-2) & 1.1e-1 (1.1e-2) & 1.0e-3 (1.7e-4) & 2000\\ 
&GANO          & 4.6e-3 (2.0e-3) & 7.4e-3 (1.5e-3) & 1.7e+0 (1.3e+0) & 3.3e-1 (8.4e-2) & 2.e-3 (1.0e-3) & 1 \\ 
\midrule
\parbox[t]{2mm}{\multirow{5}{*}{\rotatebox[origin=c]{90}{\textbf{Pop.}}}} & FFM-OT (ours) & \textbf{3.9e-5} (3.8e-5) & \textbf{7.0e-6} (9.e-6) & 4.1e+0 (5.3e+0) & 9.0e-2 (1.0e-1) & \textbf{2.7e-5} (4.6e-5) & 662\\
&FFM-VP (ours) & 6.3e-5 (4.5e-5) & \textbf{7.0e-6} (7.e-6) & \textbf{1.3e+0} (6.1e-1) & 7.8e-2 (4.5e-2) & 2.5e-3 (5.2e-4) & 494 \\
&DDPM          & 5.7e-5 (5.2e-5) & 6.0e-6 (7.0e-6) & 1.9e+0 (1.2e+0) & \textbf{5.9e-2} (4.4e-2) & 5.6e-5 (3.5e-5) & 1000\\ 
&DDO           & 1.9e-4 (8.7e-5) & 2.7e-4 (1.9e-5) & 4.2e+0 (4.1e-1) & 2.7e-1 (3.7e-2) & 3.2e-2 (1.9e-3) & 2000\\ 
&GANO          & 1.1e-3 (9.8e-4) & 4.3e-5 (7.1e-5) & 8.e+0 (2.4e+0) & 8.6e-1 (5.3e-1) & 1.6e-3 (3.6e-3) & 1\\ 
\midrule
\parbox[t]{2mm}{\multirow{5}{*}{\rotatebox[origin=c]{90}{\textbf{GDP}}}} & FFM-OT (ours) & \textbf{2.0e-5} (1.2e-5) & 9.e-6 (6.e-6) & 6.3e-1 (3.5e-1) & \textbf{3.9e-2} (1.9e-2) & \textbf{2.8e-5} (1.4e-5) & 536\\
&FFM-VP (ours) & 4.1e-5 (2.1e-5) & \textbf{8.0e-6} (7.0e-6) & \textbf{6.2e-1} (4.1e-1) & 5.0e-2 (2.5e-2) & 1.9e-4 (2.3e-5) & 494\\
&DDPM          & 1.6e-4 (1.5e-4) & 2.5e-5 (2.9e-5) & 8.6e-1 (5.9e-1) & 5.1e-2 (2.1e-2) & 1.4e-4 (1.0e-4) & 1000 \\ 
&DDO           & 2.1e-4 (1.1e-4) & 2.9e-4 (9.4e-5) & 1.7e+0 (1.1e-1) & 2.7e-1 (2.4e-2) & 9.6e-3 (1.5e-3) & 2000 \\ 
&GANO          & 8.4e-4 (7.8e-4) & 5.0e-5 (3.7e-5) & 2.6e+0 (1.3e+0) & 2.1e-1 (1.4e-1) & 1.6e-4 (1.6e-4) & 1\\ 
\midrule
\parbox[t]{2mm}{\multirow{5}{*}{\rotatebox[origin=c]{90}{\textbf{Labor}}}} & FFM-OT (ours) &\textbf{6.9e-5} (6.1e-5) & 2.6e-5 (1.1e-5) & 5.4e+0 (3.3e+0) & 1.5e-1 (1.8e-1) & \textbf{1.3e-4} (7.5e-5) & 308\\
&FFM-VP (ours) & 7.1e-5 (5.5e-5) & \textbf{2.1e-5} (9.0e-6) & \textbf{2.0e+0} (1.5e+0) & \textbf{8.6e-2} (7.3e-2) & 5.8e-4 (1.4e-4) & 302\\
&DDPM          & 4.2e-4 (3.3e-4) & 3.5e-4 (5.6e-4) & 1.8e+3 (3.5e+3) & 1.0e+1 (1.5e+1) & 2.9e-4 (1.6e-4) & 1000\\ 
&DDO           & 3.1e-4 (1.9e-4) & 4.0e-4 (1.2e-4) & 4.8e+0 (5.3e-1) & 4.3e-1 (3.9e-2) & 7.8e-3 (1.2e-3) & 2000 \\ 
&GANO          & 3.2e-3 (6.3e-3) & 6.5e-4 (4.6e-4) & 7.8e+0 (7.6e+0) & 1.2e+0 (3.7e-1) & 1.8e-3 (9.4e-4) & 1\\ 

\bottomrule
\end{tabular}
}
\label{tab:pointwise_mses_aemet}
\end{table*}

    \paragraph{Baselines}
    We compare against several functional generative models: the Denoising Diffusion Operator (DDO) \citep{lim2023score} with NCSN noise scale, GANO \citep{rahman2022generative}, and functional DDPM \citep{kerrigan2022diffusion}. We do not compare to non-functional methods, as we are primarily interested in developing discretization-invariant generative models. 
    
    All noise was specified via a Gaussian process with a tuned Matérn kernel. For the sake of a fair comparison, we used the same architecture for all models, with the exception of GANO which requires a generator and discriminator pair. We used the code provided by the authors of DDPM and GANO but re-implemented the DDO model. For all models, we performed extensive hyperparameter tuning and report the best results. Generally, we find the FFM methods are less sensitive to hyperparameter choices than the baseline methods.

    \paragraph{Results}

    Figure \ref{fig:ns_samples} shows samples from the Navier-Stokes dataset and samples generated from the various models we consider. Qualitatively, our FFM model and the DDPM model match the ground-truth samples, whereas DDO and GANO suffer from mode collapse.

    \begin{figure}[t]
    \centering
    \includegraphics[]{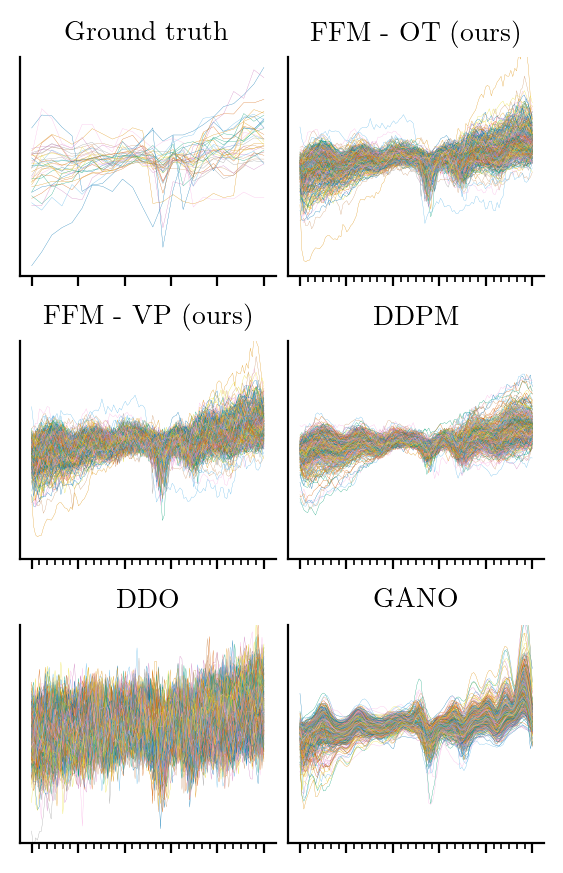}
    \caption{Samples from the Labor dataset and samples from the various models at 5x super-resolution.}
    \label{fig:genes_superres_corr}
    \end{figure}
    
    Figure \ref{fig:unconditional_aemet} shows samples from the AEMET dataset and generated samples. Our FFM model is able to qualitatively match the samples from the ground truth distribution. The DDPM samples are similar in quality, but do not respect the range of values seen in the data. For DDO, we observe smoothness issues, and for GANO, we again see mode collapse issues.
    
    Quantitatively, Table \ref{tab:pointwise_mses_aemet} evaluates model performance on the 1D datasets by computing pointwise statistics of the generated functions and computing the MSE between these pointwise statistics and those of the real data. Table \ref{tab:ns_results} reports the MSE between the density and spectra \citep{lim2023energy} of the real and generated samples on the Navier-Stokes dataset. See Appendix \ref{appendix:additional_results} for visualizations. Variants of FFM perform the best, on average, in almost all metrics considered across the wide range of domains on which we performed evaluation. While pointwise statistics have limitations, for functional models there are no clear alternatives for evaluation, and pointwise metrics are broadly used in the literature \citep{rahman2022generative, lim2023score}. Together with the qualitative results, these metrics further validate the performance of our method.

     A key benefit of the FNO architecture is the ability to perform generation at arbitrary resolutions, a necessary component in any functional task. We demonstrate this on the Labor dataset in Figure \ref{fig:genes_superres_corr}. All models are trained on the original data resolution, but samples are drawn at a five times greater resolution. Samples from FFM and DDPM qualitatively match the characteristics of the ground truth distribution, whereas samples from DDO and GANO do not match the smoothness of the original data. See Appendix \ref{appendix:additional_results} for further evalution. 
    
\paragraph{Conditional Generation}

    We also demonstrate an extension of our method for conditional tasks, such as interpolating (or extrapolating) a finite set of given observations. We explore two approaches: conditional training and a modified sampling process inspired by ILVR \citep{choi2021ilvr}. We note alternative conditional methods \citep{mathieu2023geometric} are readily applicable as well. In Figure \ref{fig:conditional}, we demonstrate these two approaches. See Appendix \ref{appendix:conditional_models} for details.

\section{Conclusion}

    We introduce Functional Flow Matching (FFM), a continuous-time normalizing flow model which allows us to model infinite-dimensional distributions. We demonstrate that FFM is able to outperform several recently proposed function-space generative models in terms of qualitative samples and pointwise metrics on a diverse set of benchmarks. Our work builds the foundations for function-space normalizing flows, and our hope is that future work may build on these foundations. In terms of limitations, FFM is implemented via the FNO \citep{li2020fourier}, which can only handle data observed on uniform grids. Exploring architectures which alleviate this assumption may increase the applicability of our methods. Additionally, there are no established benchmarks for functional generation, unlike FID \citep{heusel2017gans} for images. Developing benchmarks for these tasks is critical for future work.
    
    \begin{figure}
    \includegraphics[width=\halfwidth]{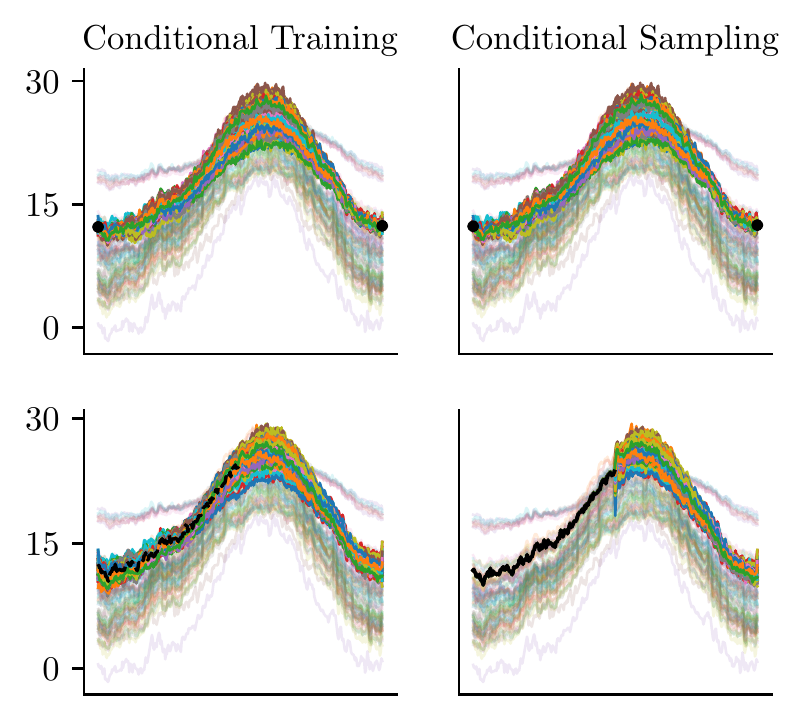}
    \caption{ \small Conditional samples from the FFM-OT model. Darker curves indicate samples and lighter curves depict real data. Conditioning information is shown in black. The first column corresponds to a conditionally trained model and the second column corresponds to a conditionally trained model in addition to conditional sampling. We see that, while the conditionally trained model takes into account the conditioning information, the conditional sampling method allows us to enforce equality of the generated samples to the conditioning information at the observation locations. }
    \label{fig:conditional}
    \end{figure}

    \paragraph{Acknowledgements.}  
    This research was supported in part by the Hasso Plattner Institute (HPI) Research Center in Machine Learning and Data Science at the University of California, Irvine, by
    NSF under awards 1900644 and 1927245, and
    by the National Institute of Health under awards R01-AG065330-02S1 and R01-LM013344.
    
\bibliographystyle{plainnat}
\bibliography{refs}

\setcounter{section}{0}
\renewcommand{\thesection}{A.\arabic{section}}
\renewcommand\theHsection{A.\thesection}

\onecolumn
\aistatstitle{Functional Flow Matching: Appendix}

\section{Proofs of Theorems}
\label{appendix:proofs}

\gluing*
\begin{proof}
    In this proof, we denote the variable of integration for integrals over $\mathcal{F}$ via a subscript on the integral. We show that for an arbitrary but fixed test function $\varphi$,
    
    \begin{align}
        \int_0^1 \int_g \partial_t \varphi(g, t) \d \mu_t(g) \d t = - \int_0^1 \int_g \langle v_t(g), \nabla_g \varphi(g, t) \d\mu_t(g) \d t.
    \end{align}
    
    To that end, we begin by analyzing the left-hand side, replacing the integration of the marginal measure $\mu_t$ with a double integral over its components:

    \begin{equation}
        \int_0^1 \int_g \partial_t \varphi(g, t) \d \mu_t(g) \d t = \int_0^1 \int_f \int_g \partial_t \varphi(g, t) \d \mu_t^{f}(g) \d \nu(f) \d t
    \end{equation}
    
    By Fubini-Tonelli and using the assumption that $v_t^{f}$ generates $\mu_t^{f}$, we obtain via the continuity equation for $(v_t^f, \mu_t^f)$:
    
    \begin{equation}
        =  - \int_f \int_0^1 \int_g \langle v_t^{f}(g), \nabla_g \varphi(g, t) \rangle \d \mu_t^{f}(g) \d t \d \nu(f) 
    \end{equation}
    
    Using our absolute continuity assumption and Fubini-Tonelli once again, we perform a change of measure to obtain
    
    \begin{align}
        &=  - \int_0^1 \int_f \int_g \langle v_t^{f}(g), \nabla_g \varphi(g, t) \rangle \left(\frac{\d \mu_t^{f} }{\d \mu_t}(g) \right) \d \mu_t(g) \d \nu(f) \d t \\
        &=  - \int_0^1  \int_f \int_g \langle v_t^{f}(g) \frac{\d \mu_t^{f} }{\d \mu_t}(g), \nabla_g \varphi(g, t) \rangle \d \mu_t (g) \d \nu(f) \d t 
    \end{align}

    Using the fact that Bochner integrals commute with inner products, an application of Fubini-Tonelli yields
    
    \begin{align}
        &=  - \int_0^1 \int_g \left\langle \int_f v_t^{f}(g) \frac{\d \mu_t^{f}}{\d \mu_t}(g) \d \nu(f), \nabla_f \varphi(f, t) \right\rangle \d \mu_t(g)\d t \\
        &=- \int_0^1 \int_g \left\langle v_t(g), \nabla_f \varphi(f, t) \right\rangle \d \mu_t(g)\d t
    \end{align}
    
    Hence, we have shown that the vector field generating $\mu_t$ is given by
    
    \begin{equation}
        v_t(g) = \int_f v_t^{f}(g) \frac{\d \mu_t^{f}}{\d \mu_t}(g) \d \nu(f).
    \end{equation}
\end{proof}

\mutuallyac*
\begin{proof}
    By assumption, there exists $\GG \subseteq \FF$ with $\nu(G) = 1$ and for any $f, g \in \GG$, we have $\mu_t^f \ll \mu_t^g$ and $\mu_t^g \ll \mu_t^f$. Fix $A \in \BB(\FF)$ with $\mu_t(A) = 0$. We claim that $\mu^f(A) = 0$ for every $f \in \GG$. Note that as $\GG \subseteq \FF$ has full measure, the integral defining $\mu_T$ may be taken over $\GG$ rather than $\FF$. Suppose for the sake of contradiction that $\mu_t^f(A) > 0$ for some $f \in \GG$. From the mutual equivalencies of the measures parametrized by $\GG$, it follows that $\mu_t^g(A) > 0$ for every $g \in \GG$. Given the form of the mixture measure $\mu_t$, it would then follow that $\mu_t(A) > 0$, which is a contradiction. Thus, $\mu_t^f \ll \mu_t$ for $\nu$-a.e. $f$ as claimed.
\end{proof}

\losstheorem*
\begin{proof}
First, note that since we are working in a real Hilbert space, for fixed $f, g \in \FF$ we have
\begin{align}
    \norm{v_t(g) - u_t(g \mid \theta)}^2 &= \langle v_t(g) - u_t(g \mid \theta), v_t(g) - u_t(g \mid \theta) \rangle \\
    &= \norm{v_t(g)}^2 + \norm{u_t(g \mid \theta)}^2 - 2 \langle v_t(g), u_t(g \mid \theta) \rangle
\end{align}
and similarly,
\begin{align}
    \norm{v_t^f(g) - u_t(g \mid \theta)}^2 = \norm{v_t^f(g)}^2 + \norm{u_t(g \mid \theta)}^2 - 2 \langle v_t(g), u_t(g \mid \theta) \rangle.
\end{align}

The first term in both is independent of the model parameters $\theta$. We analyze the remaining two terms. Below, we use subscripts on integrals over $\mathcal{F}$ to denote the variable of integration.

First, using the fact that $\mu_t$ is a mixture measure, 
\begin{align}
    \E_{t, \mu_t} &\left[ \norm{u_t(g \mid \theta)}^2 \right] = \\
    &=\int_0^1 \int_g \norm{u_t(g\mid\theta)}^2 \d \mu_t(g) \d t \\ 
    &= \int_0^1 \int_f \int_g \norm{u_t(g \mid \theta)}^2 \d \mu_t^f(g) \d \nu(f) \d t \\
    &= \E_{t, g \sim \mu_t^f, f \sim \nu} \left[ \norm{u_t(g \mid \theta)}^2 \right].
\end{align}

Next, using the exchangeability between Bochner integrals and inner products and Fubini-Tonelli,
\begin{align}
    &\E_{t, g \sim \mu_t} \left[ \langle v_t(g), u_t(g \mid \theta) \right] = \\
    &=\int_0^1 \int_g \langle v_t(g), u_t(g \mid \theta) \rangle \d \mu_t(g) \d t \\
    &= \int_0^1 \int_g \left\langle \int_f v_t^f(g) \frac{\d \mu_t^f}{\d \mu_t}(g) \d \nu(f), u_t(g \mid \theta) \right\rangle \d \mu_t(g) \d t \\ 
    &= \int_0^1 \int_f \int_g \langle v_t^f(g), u_t(g \mid \theta) \rangle \left(\frac{d \mu_t^f}{d \mu_t}(g)\right)  \d \mu_t(g) \d \nu(f) \d t \\
    &= \int_0^1 \int_f \int_g \langle v_t^f(g), u_t(g \mid \theta) \d \mu_t^f(g) \d\nu(f) \d t \\
    &= \E_{t, f \sim \nu, g \sim \mu_t^f} \left[ \langle v_t^f(g), u_t(g \mid \theta) \rangle \right].
\end{align}

This shows the equivalency of the two losses.
\end{proof}

\section{Experiment Details}
\label{appendix:experiment_details}

\subsection{Parametrizations}

The FM-OT and FM-VP model require specifying a variance schedule via the hyperparameter $\sigma_t^f$. In this work, we parametrize FFM-OT by setting $\sigma_{\min} = 1\mathrm{e}{-4}$. For FFM-VP, we set $\alpha_{t} = \cos \left( \frac{t + s}{1 + s} \frac{\pi}{2}\right)$ where $s = 0.08$, following a formulation similar to the cosine schedule introduced by \citet{nichol2021improved}.

Model-specific hyperparameters have been extensively fine-tuned via grid search, and we found the following parametrizations to consistently perform optimally across several domains:
\begin{itemize}
    \item DDPM: the noise schedule, following the notation of \cite{kerrigan2022diffusion}, is set to linearly interpolate between $\beta_0 = 1e-4$ and $\beta_T = 0.02$ in $T=1000$ timesteps. The code for this implementation was taken directly from the official repository\footnote{https://github.com/GavinKerrigan/functional\_diffusion}.
    \item DDO: following the notation of \cite{lim2023score}, we set the time interval to $T=10$, and the noise schedule geometrically interpolates between $\sigma_{10} = 1e-3$ and $\sigma_1 = 1$ on the 1D datasets and $\sigma_{10} = 1e-2$ and $\sigma_1 = 100$ on the 2D datasets. Sampling is performed by running their annealed Langevin dynamics algorithm with  $\epsilon = 2 \times 10^{-5}$ and $M=200$.
    \item GANO: the generator is trained every 5 epochs, and gradient penalty set to $\lambda=0.1$ in 1D and $\lambda=10$ in 2D. \citep{rahman2022generative}. The code for this implementation was taken directly from the official repository\footnote{https://github.com/neuraloperator/GANO}.
\end{itemize}

\paragraph{Model Architectures}
For FFM, DDPM, and DDO, the architecture used is the FNO implemented in the \texttt{neuraloperator} package \citep{li2020fourier, kovachki2021neural}. For GANO, we directly use the FNO-based model architectures for both the discriminator and generator implemented by \citet{rahman2022generative} for the 2D dataset, while for the 1D datasets we use the same FNO architecture as the other methods.

\paragraph{Gaussian Measures}
Each model experimented with relies on noise sampled from a Gaussian measure. In our work, we consider a mean-zero Gaussian process (GP) parametrized by a Matérn kernel with $\nu = 1/2$. In 1D, the kernel hyperparameters are set to have a variance $\sigma^2 = 0.1$ and length scale $\ell = 1e-2$. In 2D, the variance is $\sigma^2 = 1$ and the length scale was set to $\ell=1e-3$ for DDO and GANO, and $\ell=1e-2$ for FFM and DDPM.

\paragraph{Training}
All models are trained using the Adam optimizer. In 1D, we use an initial learning rate of $1\mathrm{e}{-3}$, scheduled to decrease by one order of magnitude after 50 epochs for all datasets but AEMET, where it is decreased every 25 epochs. In 2D, we use an initial learning rate of $5e-4$ for FFM, DDPM, and DDO, and an initial learning rate of $1e-4$ for GANO, and this initial learning rate was decayed by one order of magnitude every 25 epochs. 

\subsection{Dataset Details}
\label{appendix:datasets}

\paragraph{Navier-Stokes.} This dataset consists of solutions to the Navier-Stokes equations on a 2D torus at a resolution of 64x64. For the sake of efficient training, we randomly selected 20,000 datapoints for training from the original dataset \citep{li2022learning} as there is a high degree of redundancy in the data. For FFM, DDPM, and DNO, we use 4 Fourier layers of 32 modes and 64 hidden channels, 256 lifting channels, 256 projection channels, and the GeLU activation function \citep{hendrycks2016gaussian}. For GANO, we use 32 modes and set the number of hidden channels to 16 due to memory constraints. All models were trained for 300 epochs at a batch size of 128.

\paragraph{AEMET dataset.} This dataset consists of a set of functions describing the mean curve of the average daily temperature (in Celsius) for the period 1980-2009 recorded by 73 weather stations in Spain \citep{febrero2012statistical}. Each function is observed on a uniform grid at a resolution of 365. The neural architecture we use for this dataset is an FNO with a width of 256 and 64 modes, kept constant for all models considered in this experiment. The model is trained for 50 epochs, with batch size set to 73. 

\paragraph{Gene expression.} The original dataset consists of 10,928 time series at 20 uniformly spaced time points, recording the amplitude of gene expression for 4 different genes. The genes are concatenated to create the visual effect of spikes occurring periodically in time, while maintaining the structure of the original dataset. The data was log-transformed and centered before being fed to the model. We restrict our focus to a subset of 156 functions exhibiting large gene expression, determined by the standard deviation averaged across time for each centered function being greater than 0.3. For this dataset, we use an FNO with a width of 256 and 16 modes across all models. The model is trained for 200 epochs, with batch size set to 16.

\paragraph{Economic datasets.} The first two datasets are taken from the Maddison Project database \citep{bolt2020maddison, inklaar2018rebasing}, and the third from the IMF \citep{IMFlabor}. The datasets were picked specifically for their distinct visual characteristics, explored in greater detail in Appendix \ref{appendix:additional_results}.

\begin{itemize}
    \item \textit{Population}: time series of the evolution of the population for 169 countries across the globe from the year 1950 to 2018 (that is, discretized at 69 points in time). For a clearer visual representation, each time series was divided by its mean, so each curve represents the population for each country, relative to the mean population for that country over the 69 years under consideration. The functions in this dataset exhibit linear growth over time, with a change point shared across observations.
    \item \textit{GDP}: time series representing the evolution of GDP per capita from 1950 to 2018. The original dataset consists of 169 countries, but time series presenting missing values were removed yielding 145 observations. The same preprocessing as that described above was applied to the data. While the functions seem to exhibit the same change point as that observed for the population datasets, the growth over time is noisier and exhibits irregular patterns.
    \item \textit{Labor}: size of labor force per quarter between Q1 2017 and Q4 2022 (for a total of 24 points in time), for a subset of 35 countries (obtained removing those with missing values from the original 105 observations).  The same preprocessing as that described for the population dataset was applied to the data. This dataset tests the ability of the generative models to learn from small and multimodal data.
\end{itemize}
The models for the population and GDP datasets have width set to 256 and 32 modes, while the one for the labor dataset has width set to 128 and 8 modes. All models were trained for 100 epochs, with a batch size of 16. 

\subsection{Sampling Details}

We use the \texttt{torchdiffeq} \citep{torchdiffeq} package for all ODE solvers. The specific solver we use is \texttt{dopri5}, an implementation of the Dormand-Prince method \citep{DORMAND198019} of order 5. 
We set the absolute and relative tolerance parameters to $1\mathrm{e}{-10}$ for the 1D datasets and $1\mathrm{e}{-5}$ for the 2D datasets. Note that setting such a tolerance gives us an explicit way of trading off sample quality for sampling efficiency.

\section{Additional Experimental Results}
\label{appendix:additional_results}

This section contains several additional experimental results and figures. First, in Section \ref{appendix:addl_mogp} we explore an additional synthetic dataset consisting of a mixture of Gaussian processes (MoGP). Next, we provide additional results on super-resolution in Section \ref{appendix:superres}. Sections \ref{appendix:addl_1d_results} and \ref{appendix:addl_ns} provide further visualizations of our results on the 1D and 2D datasets considered in the main paper respectively.

\subsection{Additional Results: MoGP}
\label{appendix:addl_mogp}

\paragraph{Mixture of Gaussian Processes (MoGP) Dataset.}
The experiment considers the task of generating samples from a mixture of two GPs. The two components, with equal weights, have mean functions $m_1 = 10x - 5$ and $m_2 = -10x + 5 $, and a squared-exponential kernel with variance $\sigma^2 = 0.04$ and length scale $\ell = 0.1$. The synthetic samples used for training are observed on a uniform grid at a resolution of 64 on the interval $[0,1]$. All models were trained on the same sample of $N=5000$ realizations of the mixture of GPs. Figure \ref{fig:appendix_mogp_detailed} illustrates a visual comparison of 500 samples from each model, while Table \ref{tab:pointwise_mses_mogp} presents a quantitative comparison of the mean squared error (MSE) on various pointwise statistics. Additionally, Figure \ref{fig:appendix_mogp_detailed} provides a comprehensive depiction of the variations in these pointwise statistics across different models. 

\begin{table*}[!ht]
    \centering
    \caption{  
    Average MSEs between true and generated samples for various pointwise statistics on the MoGP dataset, along with the standard deviation (across ten random seeds). The average number of function evaluations (NFEs) for each model is also reported. Variants of our proposed FFM model obtain the best or second best average performance across all metrics. DDPM outperforms FFM in terms of variance, but only by a small margin.}
    \resizebox{\textwidth}{!}{%
    \tiny
    \begin{tabular}{lccccccc}
    \toprule
     & Mean & Variance & Skewness & Kurtosis & Autocorrelation & NFEs\\
    \midrule
    FFM-OT (ours)                       & \textbf{2.2e-2} (3.e-2) & 2.9e-1 (3.2e-1) & 1.6e-2 (1.1e-2) & \textbf{1.1e-2} (1.2e-2) & \textbf{7.e-6} (6.e-6) & 740\\
    FFM-VP (ours)                       & 3.9e-2 (3.6e-2) & 3.6e-1 (5.6e-1) & \textbf{1.4e-2} (5.2e-3) & 1.5e-2 (1.2e-2) & 8.e-6 (8.e-6) & 716\\
    DDPM                               & 3.0e-2 (2.4e-2) & \textbf{1.4e-1 }(1.9e-1) & 1.5e-2 (9.6e-3) & 1.2e-2 (8.1e-3) & 1.9e-5 (2.2e-5) & 1000\\ 
    DDO                               & 7.3e-1 (9.6e-1) & 2.7e+0 (5.3e+0) & 4.2e-1 (8.7e-1) & 2.8e-1 (3.9e-1) & 1.3e-5 (8.e-6) & 2000\\ 
    GANO                               & 1.9e-1 (1.6e-1) & 8.1e+0 (6.0e+0) & 3.4e-1 (2.5e-1) & 4.6e-2 (3.9e-2) & 6.2e-4 (6.7e-4) & 1
  \\ 
    \bottomrule
    \end{tabular}
    }
    \label{tab:pointwise_mses_mogp}
\end{table*}

\begin{figure*}[!ht]%
    \centering
    \subfloat{\includegraphics[width= 0.2\textwidth]{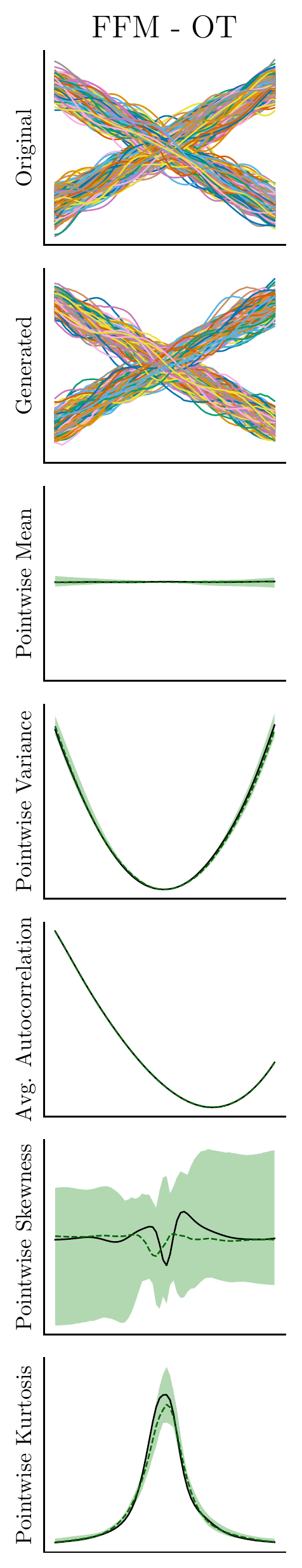}}%
    \subfloat{\includegraphics[width= 0.2\textwidth]{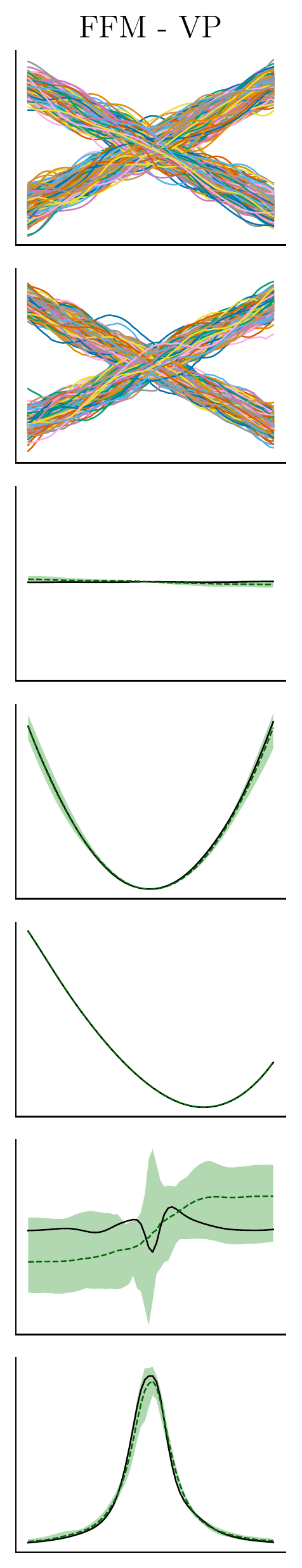}}%
    \subfloat{\includegraphics[width= 0.2\textwidth]{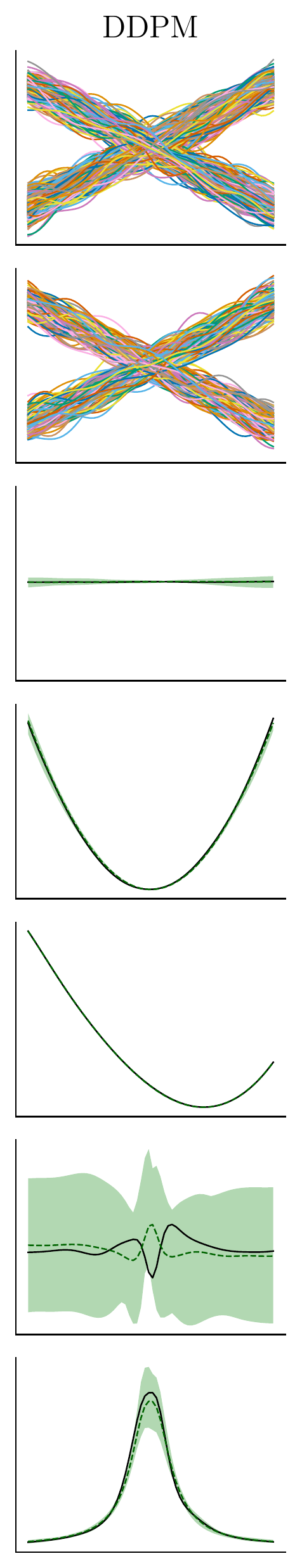}}%
    \subfloat{\includegraphics[width= 0.2\textwidth]{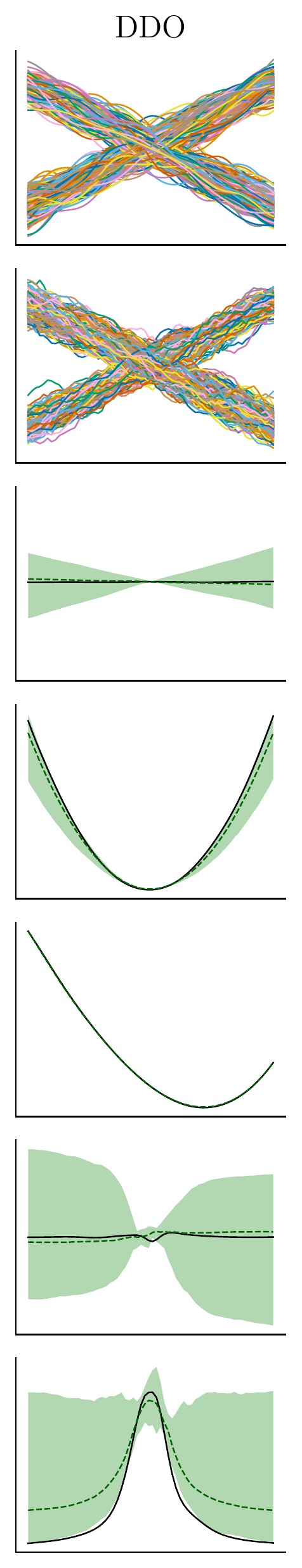}}%
    \subfloat{\includegraphics[width= 0.2\textwidth]{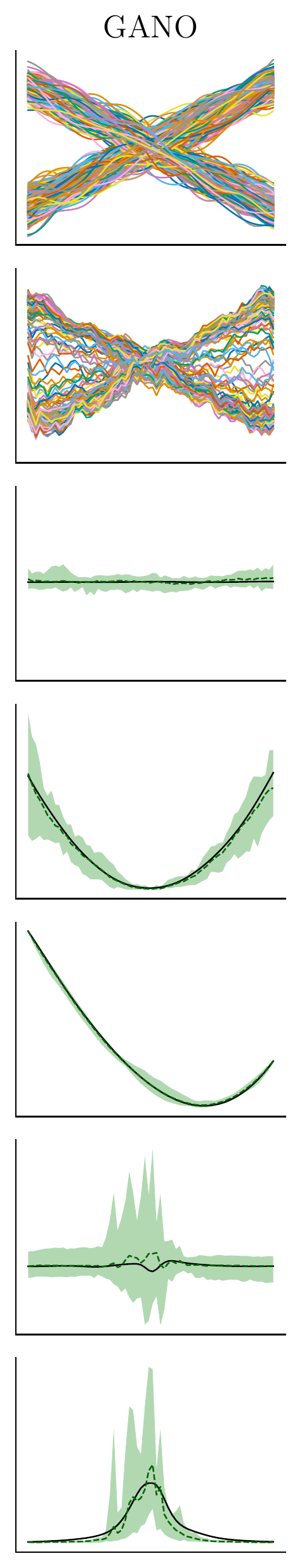}}%
    \caption{Various pointwise statistics on the MoGP dataset. Curves in black indicate the corresponding pointwise statistic for the original dataset (top row). The green error bands represent the minimal and maximal value of the pointwise statistic from 500 samples across ten random seeds of the corresponding model. The dashed green lines indicate the mean pointwise statistic across these ten runs for each model. See Table \ref{tab:pointwise_mses_mogp} for a quantitative comparison.}%
    \label{fig:appendix_mogp_detailed}%
\end{figure*}

\clearpage

\subsection{Additional Results: Super-Resolution}
\label{appendix:superres}

    Here, we provide additional visualizations regarding the ability of all models considered to perform super-resolution, i.e. to sample at a resolution greater than the training dataset. All models considered are trained at the original dataset resolution, but due to the neural operator architectures being used, we may sample at arbitrary resolutions. We note that quantitatively evaluating these super-resolved samples is difficult as we do not have access to a notion of higher-resolution ground truth here. 
    
    Figure \ref{fig:genes_superres} shows samples from the original Gene expression time series dataset \citep{orlando2008global} at a resolution of 20, as well as samples from each model at a 5x resolution, i.e. a resolution of 100. We see that, qualitatively, samples from FFM, DDPM, and GANO resemble those of the original dataset. The samples from DDO appear overly rough, and the samples from GANO are smoother than those from FFM and DDPM. 
    
    Figure \ref{fig:econ_superres} shows qualitatively similar results on the econometrics datasets \citep{bolt2020maddison, inklaar2018rebasing, IMFlabor}, with the exception of FFM-VP generating samples that are rougher than the original data. To further explore the quality of these super-resolved samples, we additionally provide the correlation matrices of the original data and super-resolved samples in Figure \ref{fig:econ_superres_corr}. We generally see that FFM-OT, FFM-VP, DDPM, and GANO are able to qualitatively capture the original correlation structures, whereas DDO fails to do so. We additionally note that on the Population and GPD datasets, FFM-VP, DDO, and GANO display a consistent strong diagonal band, indicating that these models generate samples at a variance which is too large when super-resolved. All models display this failure mode on the Labor dataset.

    \begin{figure*}[!ht]
        \centering
        \includegraphics[width=\textwidth]{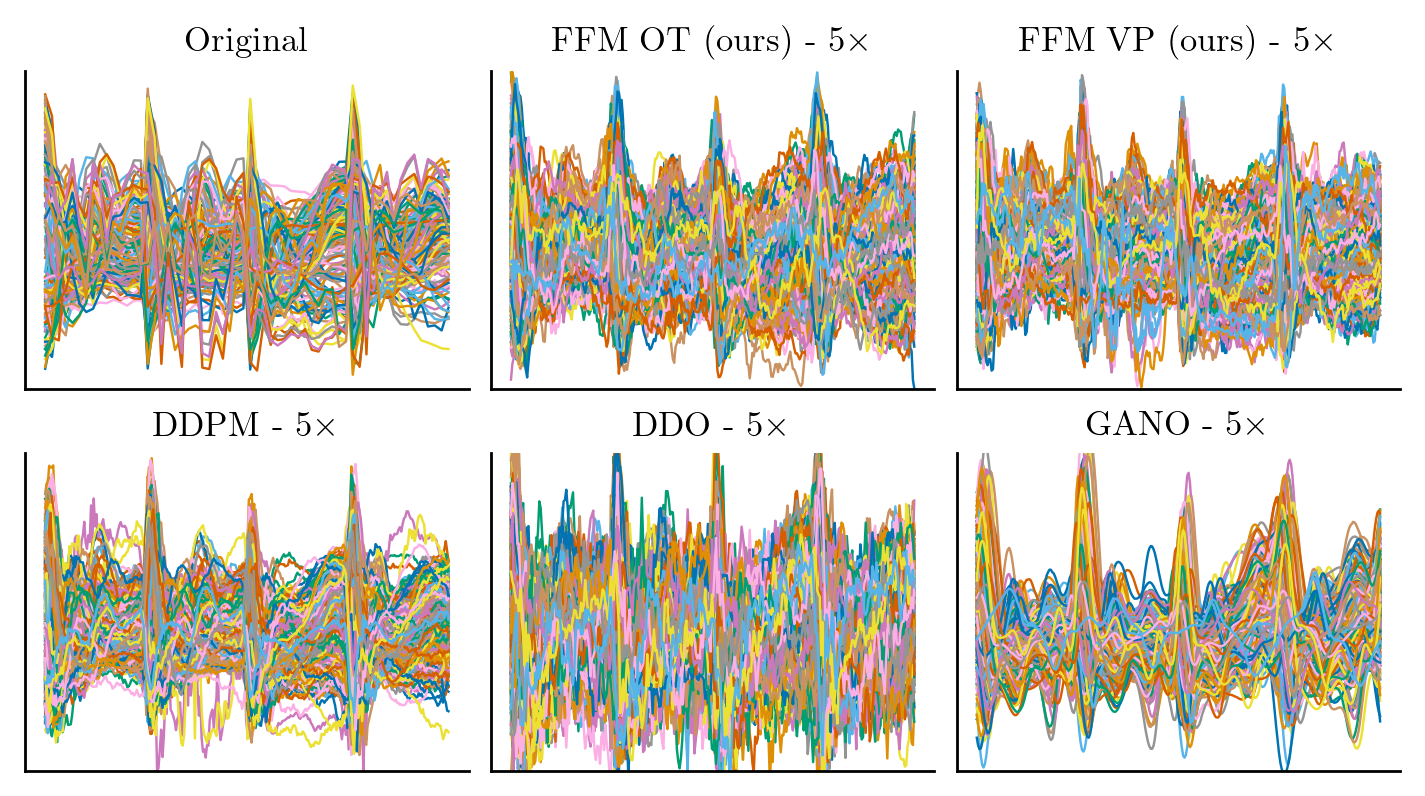}
        \caption{Samples from the gene expression dataset and samples from the various models at a 5x super-resolution.}
        \label{fig:genes_superres}
    \end{figure*}
    
    \begin{figure*}[!ht]
        \centering
        \includegraphics[width=\textwidth]{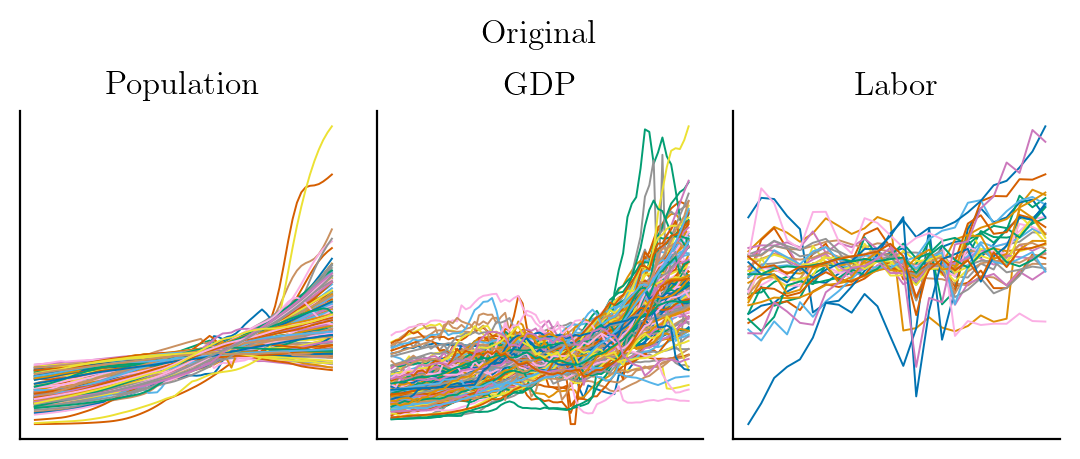}
        \includegraphics[width=\textwidth]{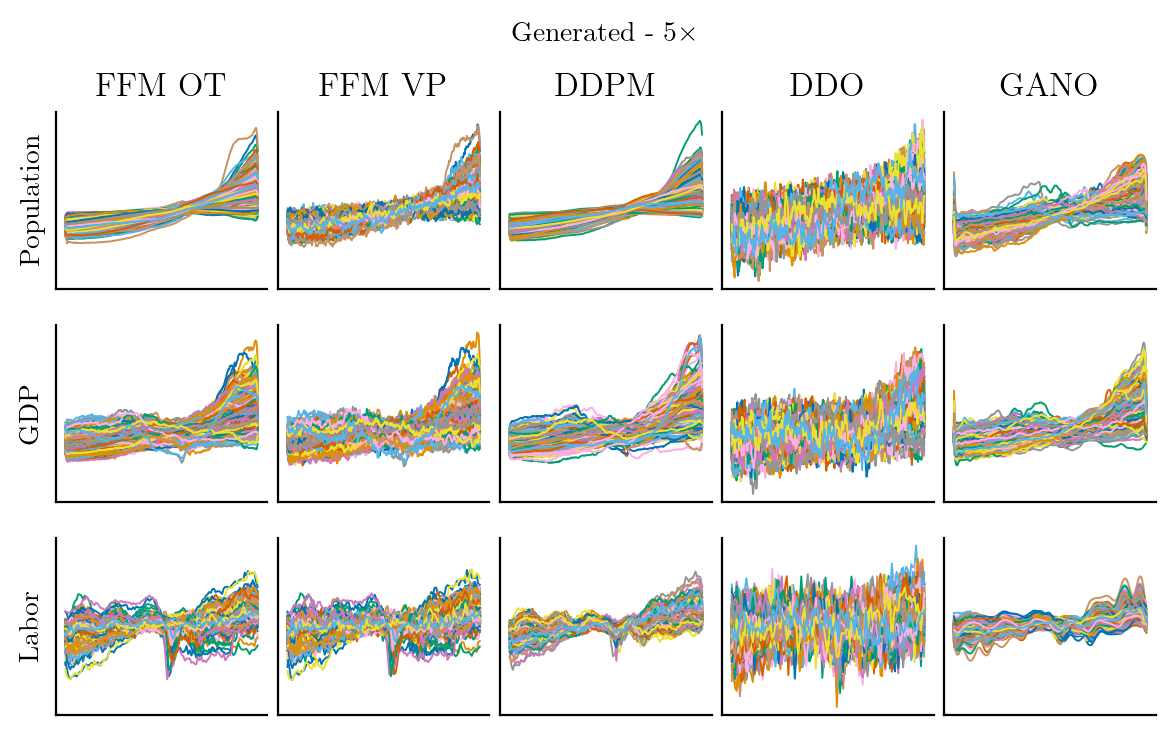}
        \caption{Samples from the three economics datasets and samples from the various models at a 5x super-resolution.}
        \label{fig:econ_superres}
    \end{figure*}
    
    \begin{figure*}[!ht]
    \centering
        \includegraphics[width=0.9\textwidth]{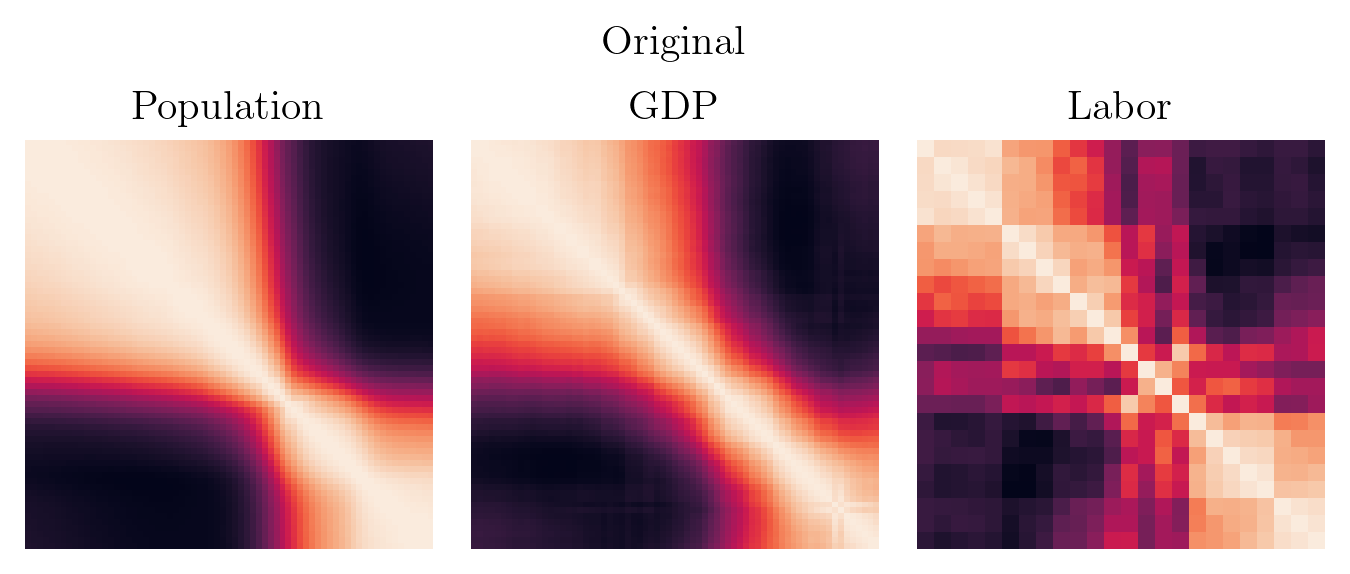}
        \includegraphics[width=0.95\textwidth]{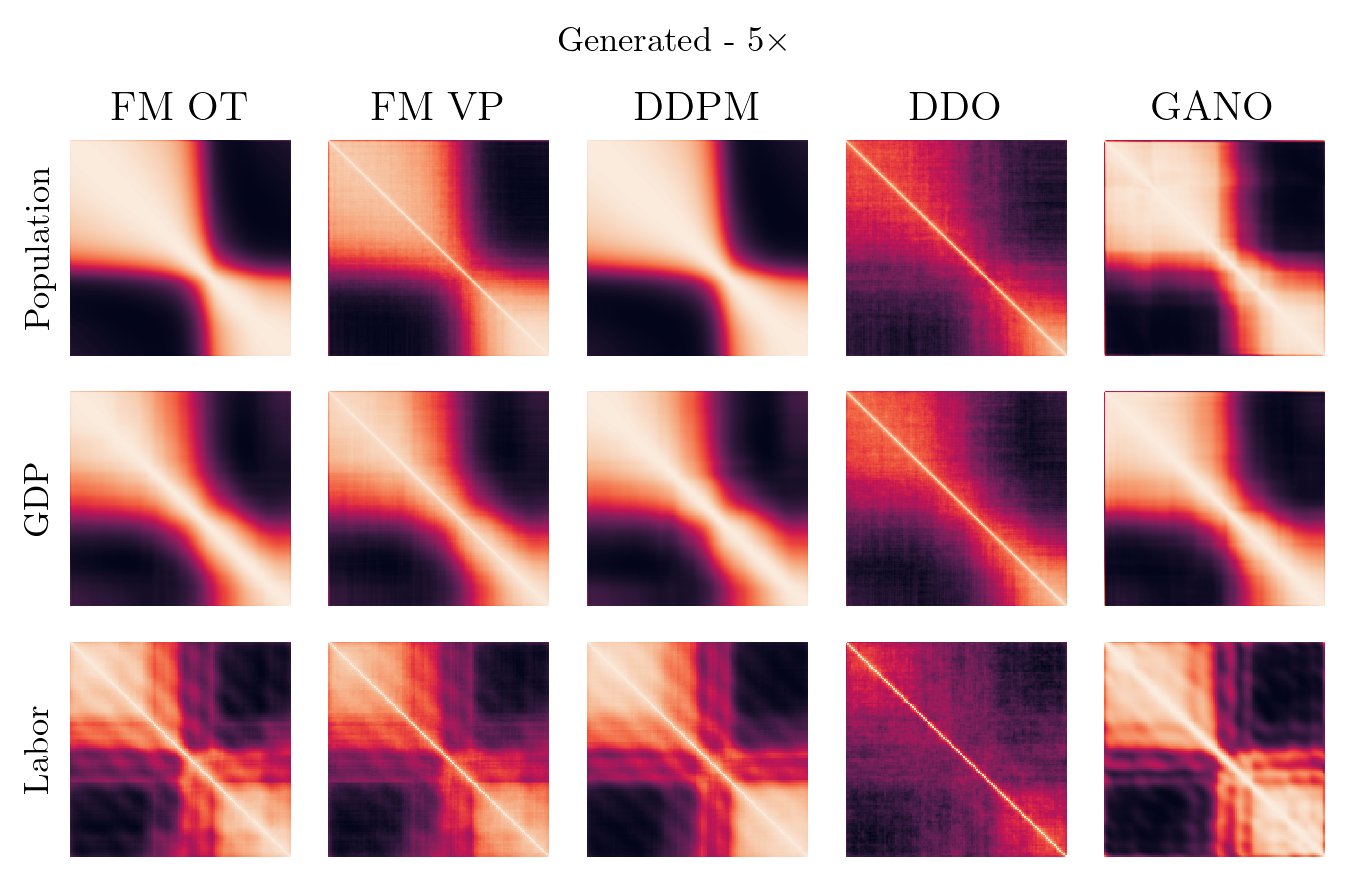}
        \caption{Correlation matrices for the three economics datasets (top row), as well as correlation matrices for each dataset for the generated samples from each model at a 5x super-resolution.}
        \label{fig:econ_superres_corr}
    \end{figure*}

\clearpage
\subsection{Additional Results: 1D Datasets}
\label{appendix:addl_1d_results}

    This section provides a further analysis and visualization of the generated samples for the 1D datasets considered in the main paper. In Figures \ref{fig:appendix_aemet_detailed} through \ref{fig:appendix_genes_detailed}, we plot the original data, generated samples from each model, and various pointwise statistics for the real and generated samples. Curves in black indicate the corresponding pointwise statistic for the original dataset (top row). The green error bands represent the minimal and maximal value of the pointwise statistic from 500 samples across ten random seeds of the corresponding model. The dashed green lines indicate the mean pointwise statistic across these ten runs for each model. See Table \ref{tab:pointwise_mses_aemet} for a quantitative comparison derived from these figures.
    
    \begin{figure*}[t]%
        \centering
        \subfloat{\includegraphics[width= 0.2\textwidth]{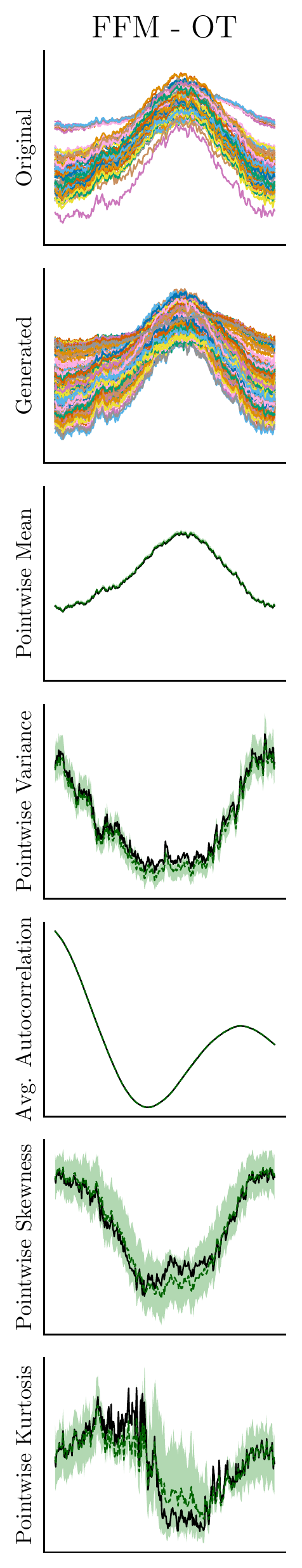}}%
        \subfloat{\includegraphics[width= 0.2\textwidth]{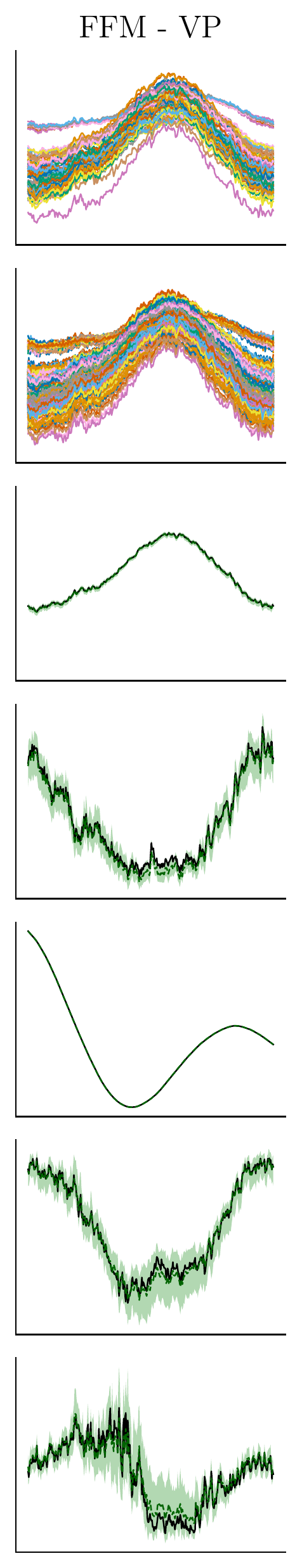}}%
        \subfloat{\includegraphics[width= 0.2\textwidth]{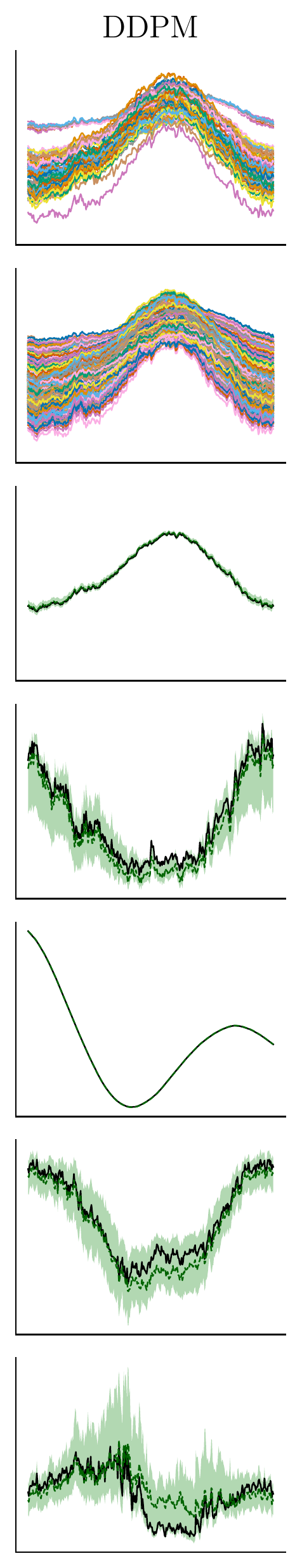}}%
        \subfloat{\includegraphics[width= 0.2\textwidth]{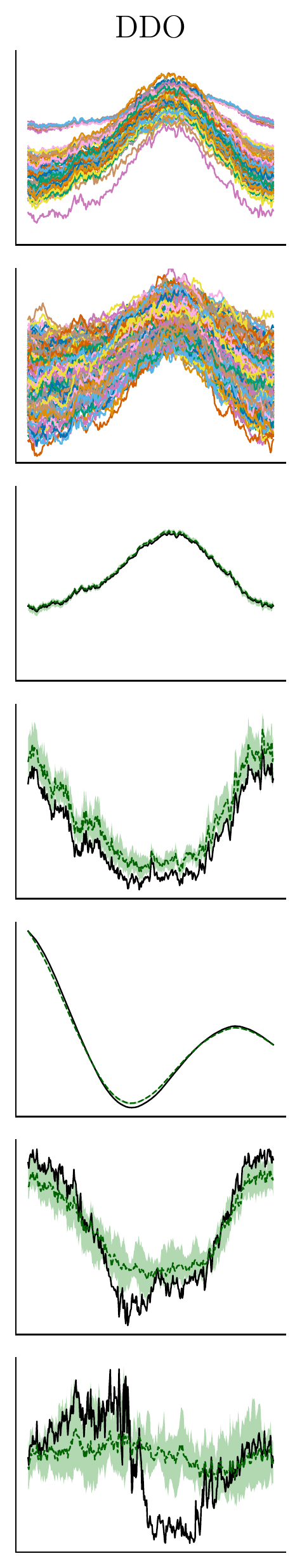}}%
        \subfloat{\includegraphics[width= 0.2\textwidth]{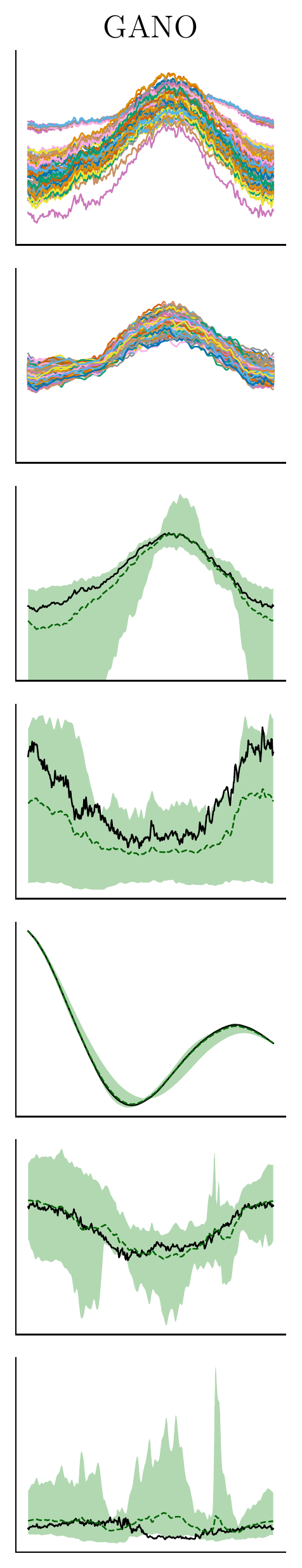}}%
        \caption{Various pointwise statistics on the AEMET dataset.}%
        \label{fig:appendix_aemet_detailed}%
    \end{figure*}

    \begin{figure*}[!ht]%
    \centering
    \subfloat{\includegraphics[width= 0.2\textwidth]{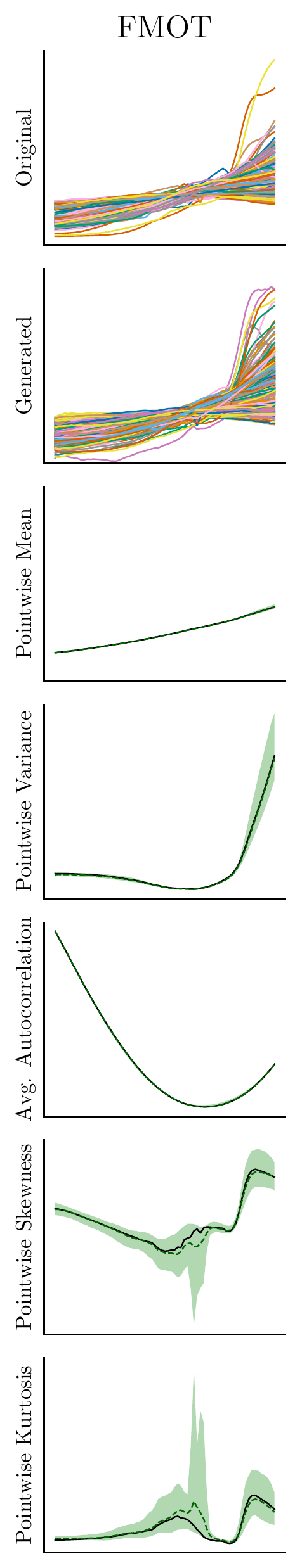}}%
    \subfloat{\includegraphics[width= 0.2\textwidth]{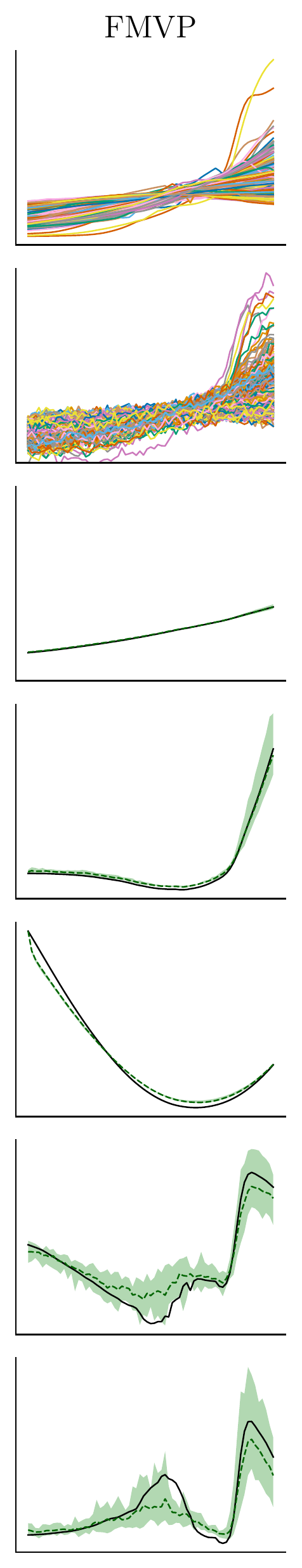}}%
    \subfloat{\includegraphics[width= 0.2\textwidth]{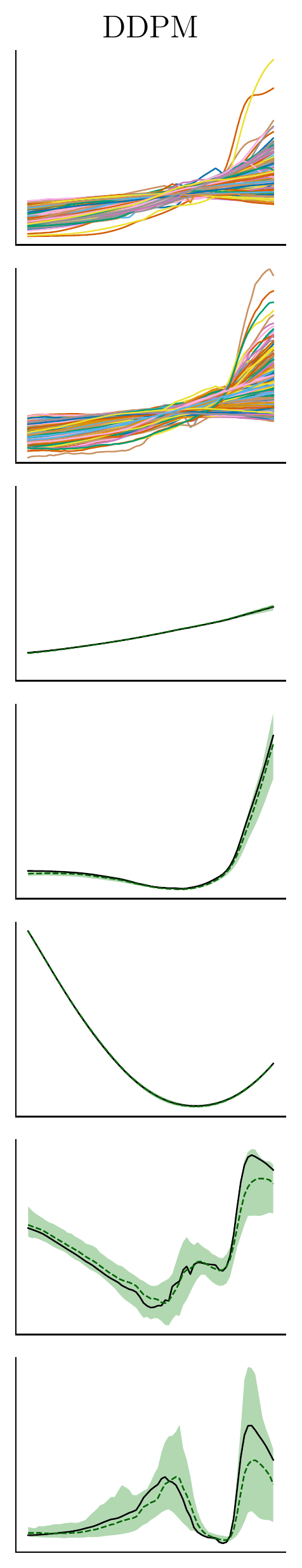}}%
    \subfloat{\includegraphics[width= 0.2\textwidth]{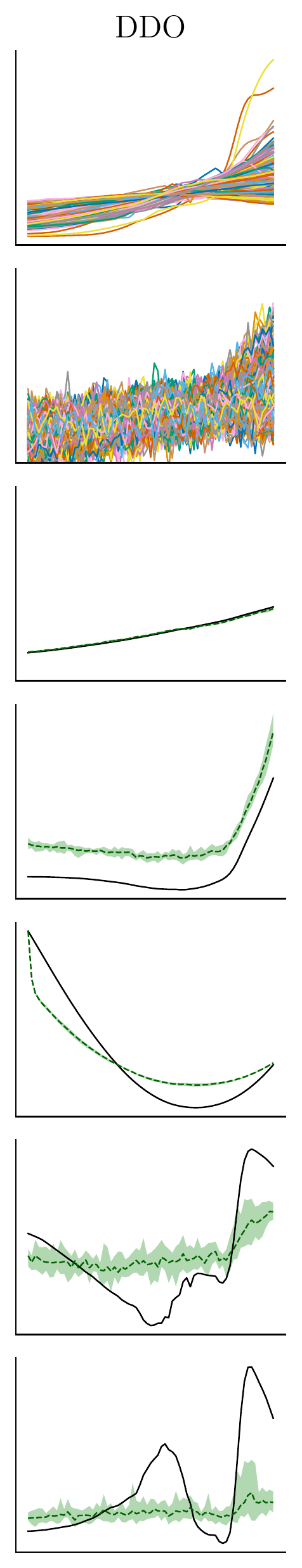}}%
    \subfloat{\includegraphics[width= 0.2\textwidth]{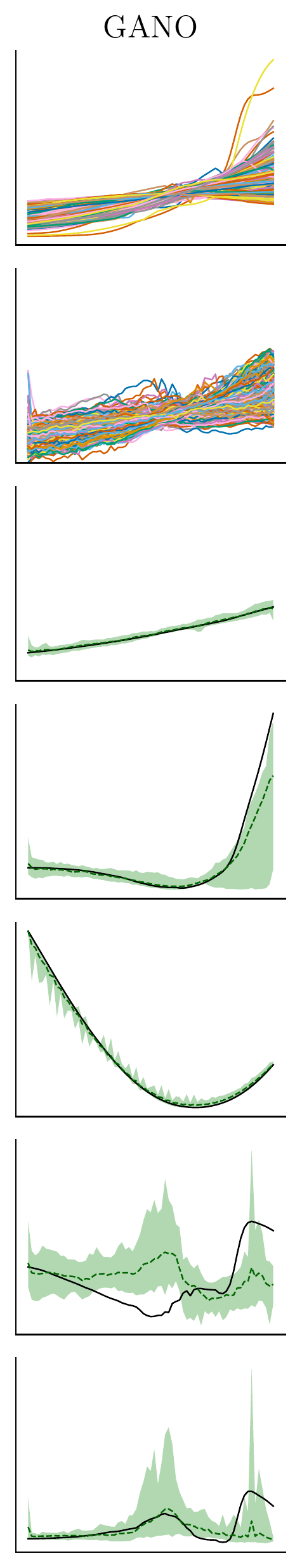}}%
    \caption{Various pointwise statistics on the Population dataset.}%
    \label{fig:appendix_pop_detailed}%
    \end{figure*}
    \begin{figure*}[!ht]%
    \centering
    \subfloat{\includegraphics[width= 0.2\textwidth]{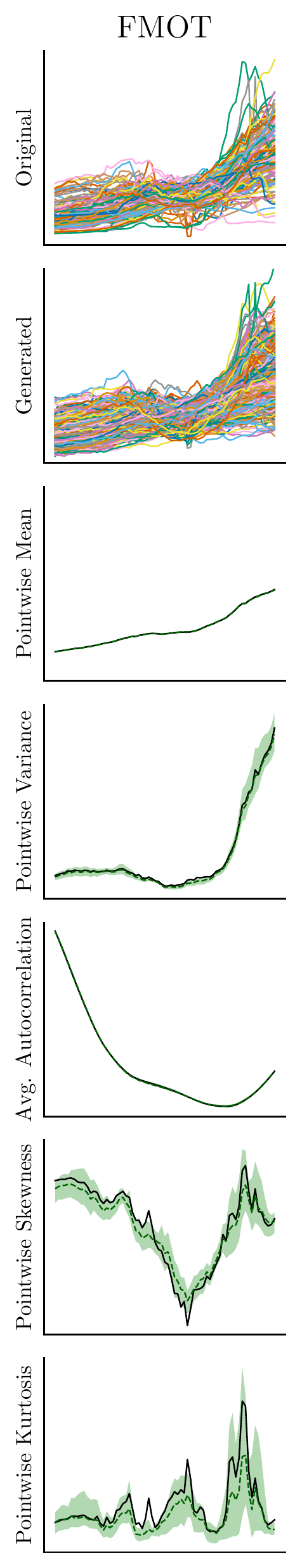}}%
    \subfloat{\includegraphics[width= 0.2\textwidth]{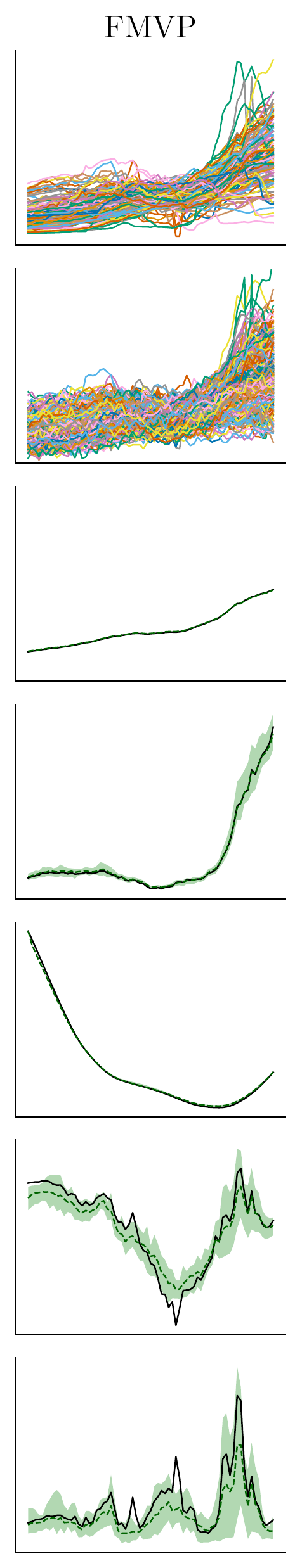}}%
    \subfloat{\includegraphics[width= 0.2\textwidth]{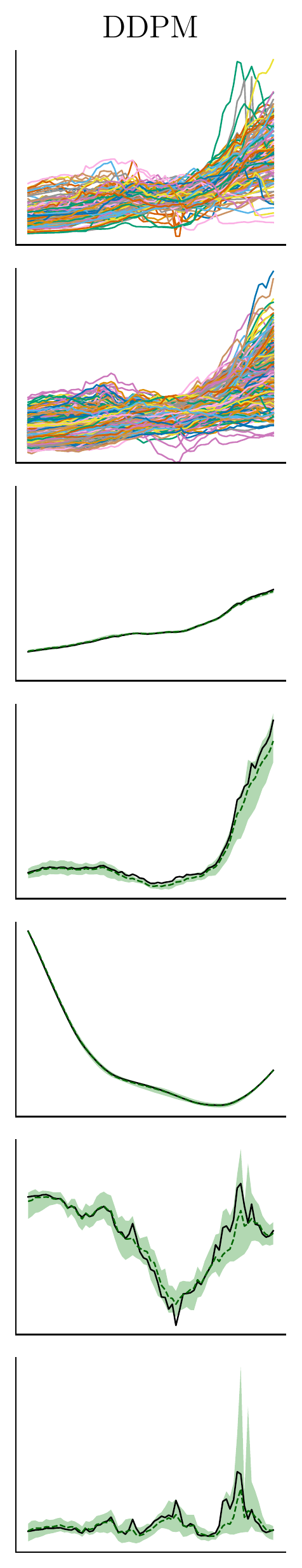}}%
    \subfloat{\includegraphics[width= 0.2\textwidth]{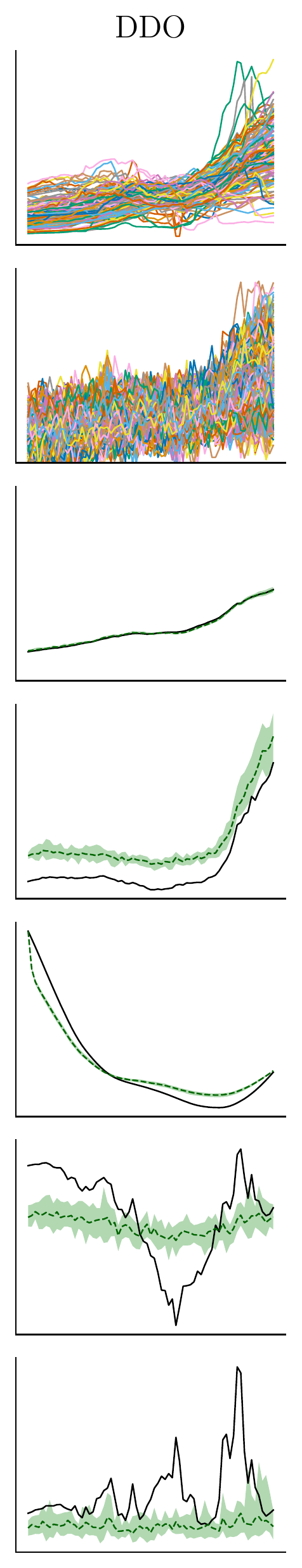}}%
    \subfloat{\includegraphics[width= 0.2\textwidth]{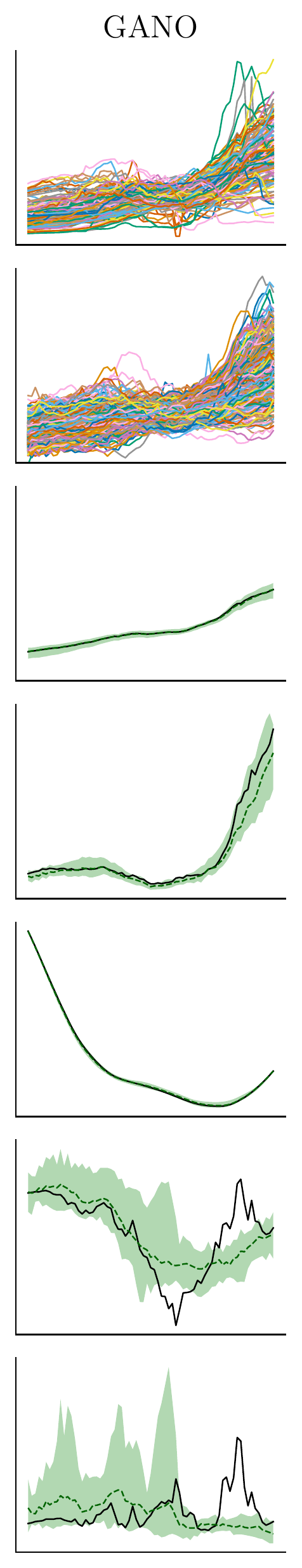}}%
    \caption{Various pointwise statistics on the GDP dataset.}%
    \label{fig:appendix_gdp_detailed}%
    \end{figure*}
    \begin{figure*}[!ht]%
    \centering
    \subfloat{\includegraphics[width= 0.2\textwidth]{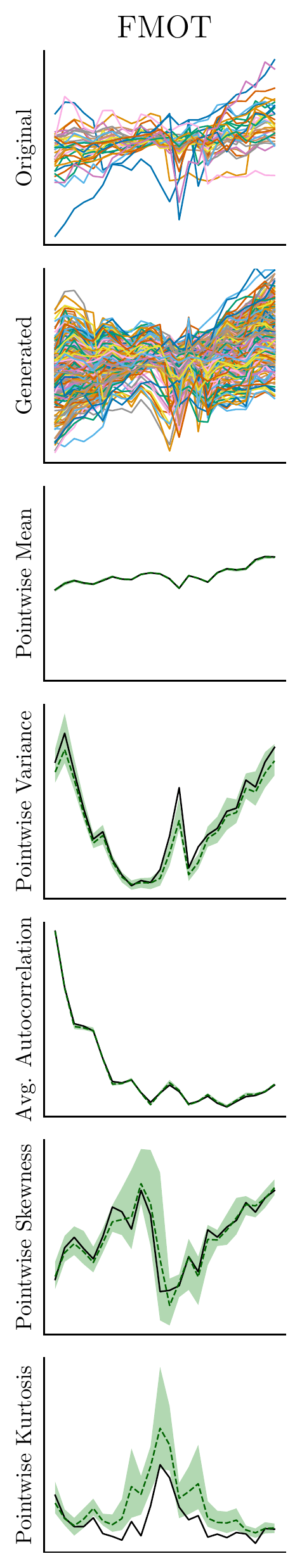}}%
    \subfloat{\includegraphics[width= 0.2\textwidth]{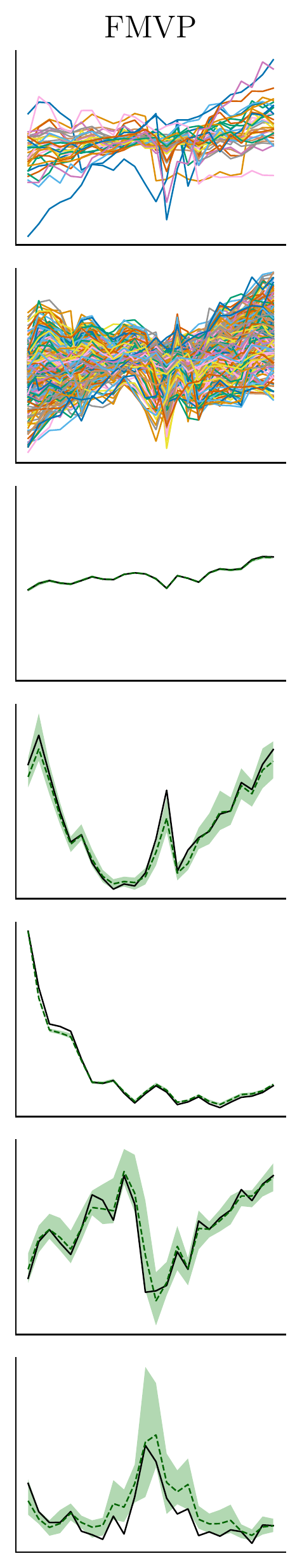}}%
    \subfloat{\includegraphics[width= 0.2\textwidth]{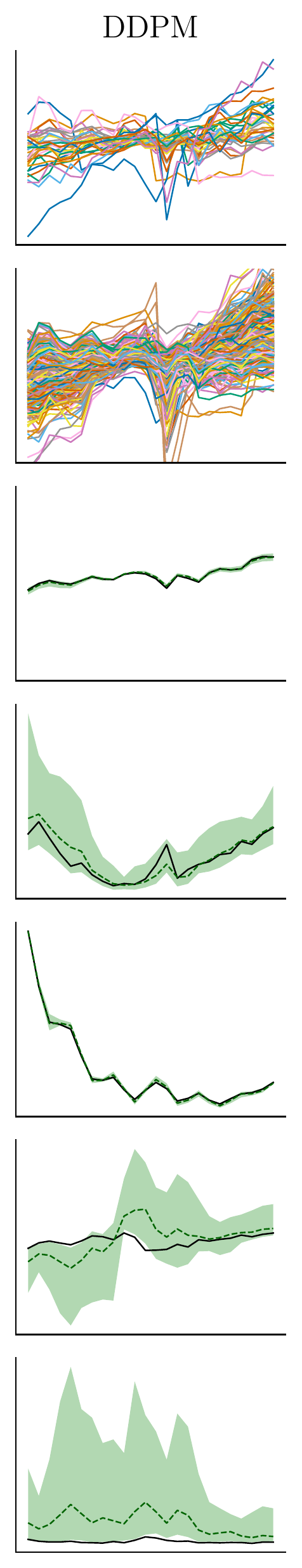}}%
    \subfloat{\includegraphics[width= 0.2\textwidth]{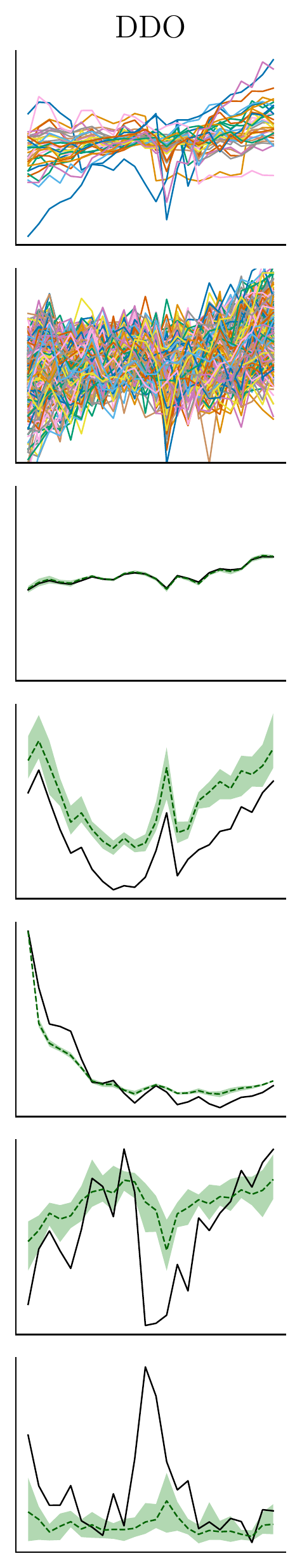}}%
    \subfloat{\includegraphics[width= 0.2\textwidth]{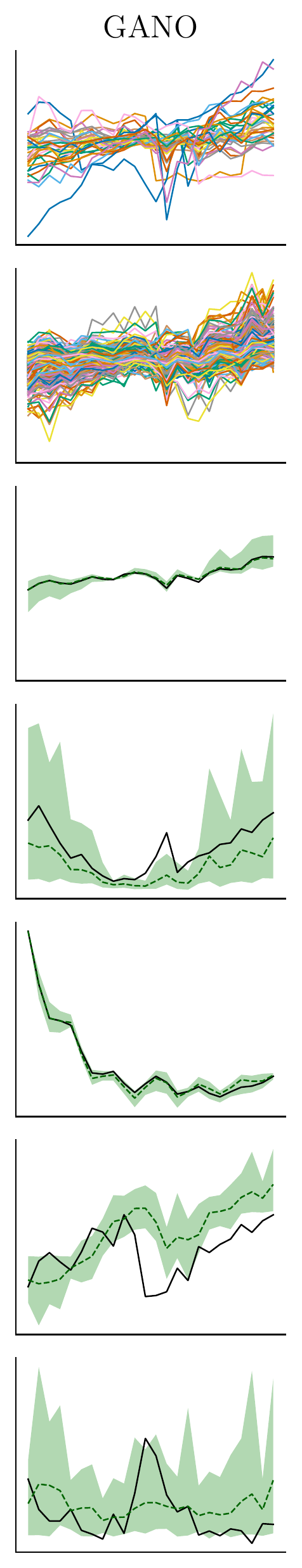}}%
    \caption{Various pointwise statistics on the Labor dataset.}%
    \label{fig:appendix_labor_detailed}%
    \end{figure*}

    \begin{figure*}[!ht]%
    \centering
    \subfloat{\includegraphics[width= 0.2\textwidth]{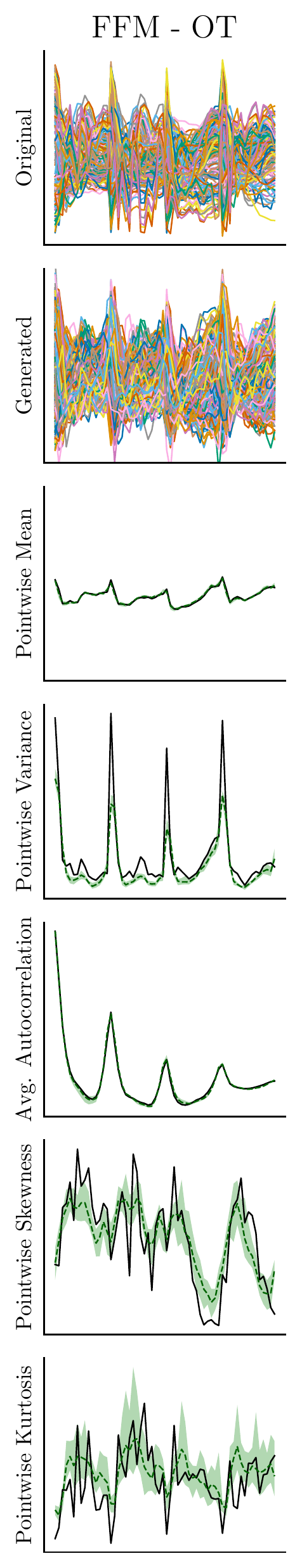}}%
    \subfloat{\includegraphics[width= 0.2\textwidth]{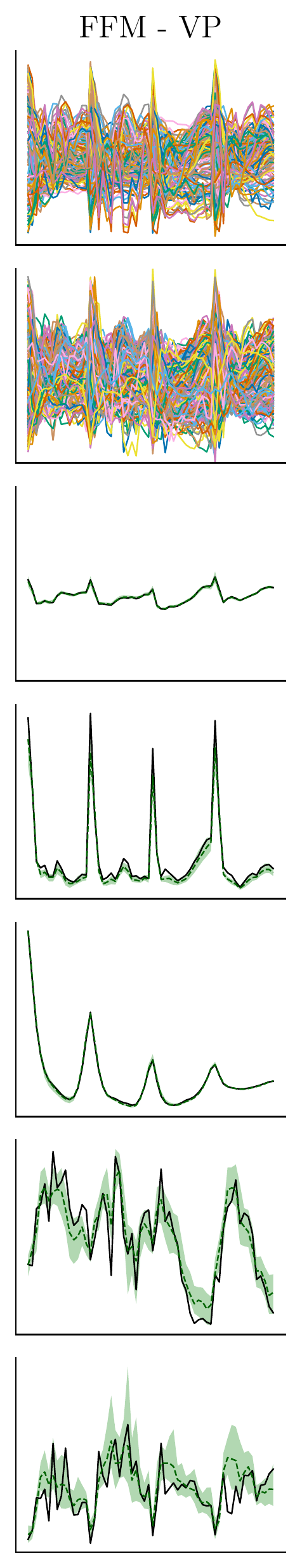}}%
    \subfloat{\includegraphics[width= 0.2\textwidth]{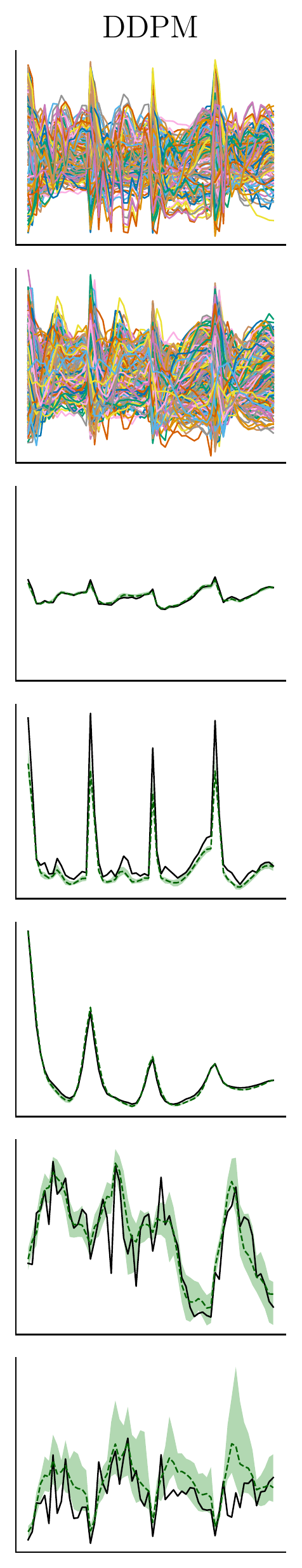}}%
    \subfloat{\includegraphics[width= 0.2\textwidth]{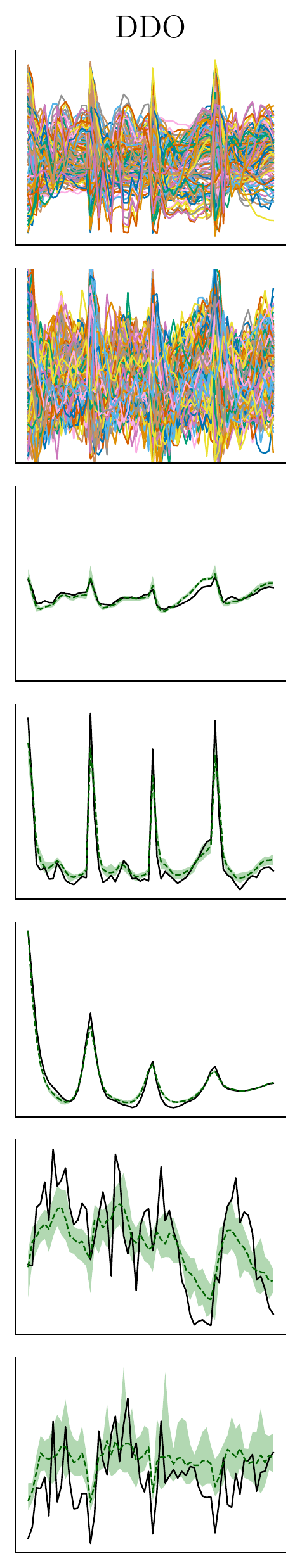}}%
    \subfloat{\includegraphics[width= 0.2\textwidth]{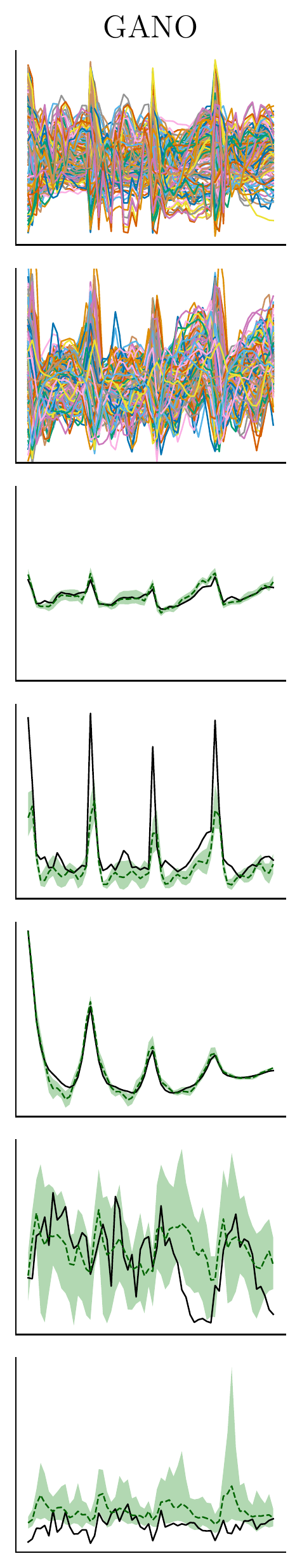}}%
    \caption{Various pointwise statistics on the Genes dataset.}%
    \label{fig:appendix_genes_detailed}%
    \end{figure*}

\clearpage
\subsection{Additional Results: Navier-Stokes Dataset}
\label{appendix:addl_ns}

   In this section, we provide additional visualizations and evaluation on the Navier-Stokes dataset, corresponding to the samples in Figure \ref{fig:ns_samples} and Table \ref{tab:ns_results} in the main paper. The first row of Figure \ref{fig:ns_density_spectrum} plots a Gaussian KDE with a fixed bandwidth of 0.5 for the pixel-wise values of both the real and generated samples across all methods. We observe that FFM and DDPM closely match the ground-truth distribution, whereas DDO places too much mass around zero, and GANO learns a multimodal distribution. In the second row, we plot the spectrum of both the real and generated samples, i.e. the log-energy as a function of the wavenumber. We see that FFM and DDPM closely match the true spectrum for low wavenumbers, whereas the fits of DDO and GANO are less close. For all models, the generated samples fail to match the true spectrum at high wavenumbers. We obtain quantiative metrics from these visualizations by considering the pointwise MSE between the ground truth and generated curves to obtain the metrics in Table \ref{tab:ns_results}.

    \begin{figure}[h]
        \centering
        \includegraphics[width=\textwidth]{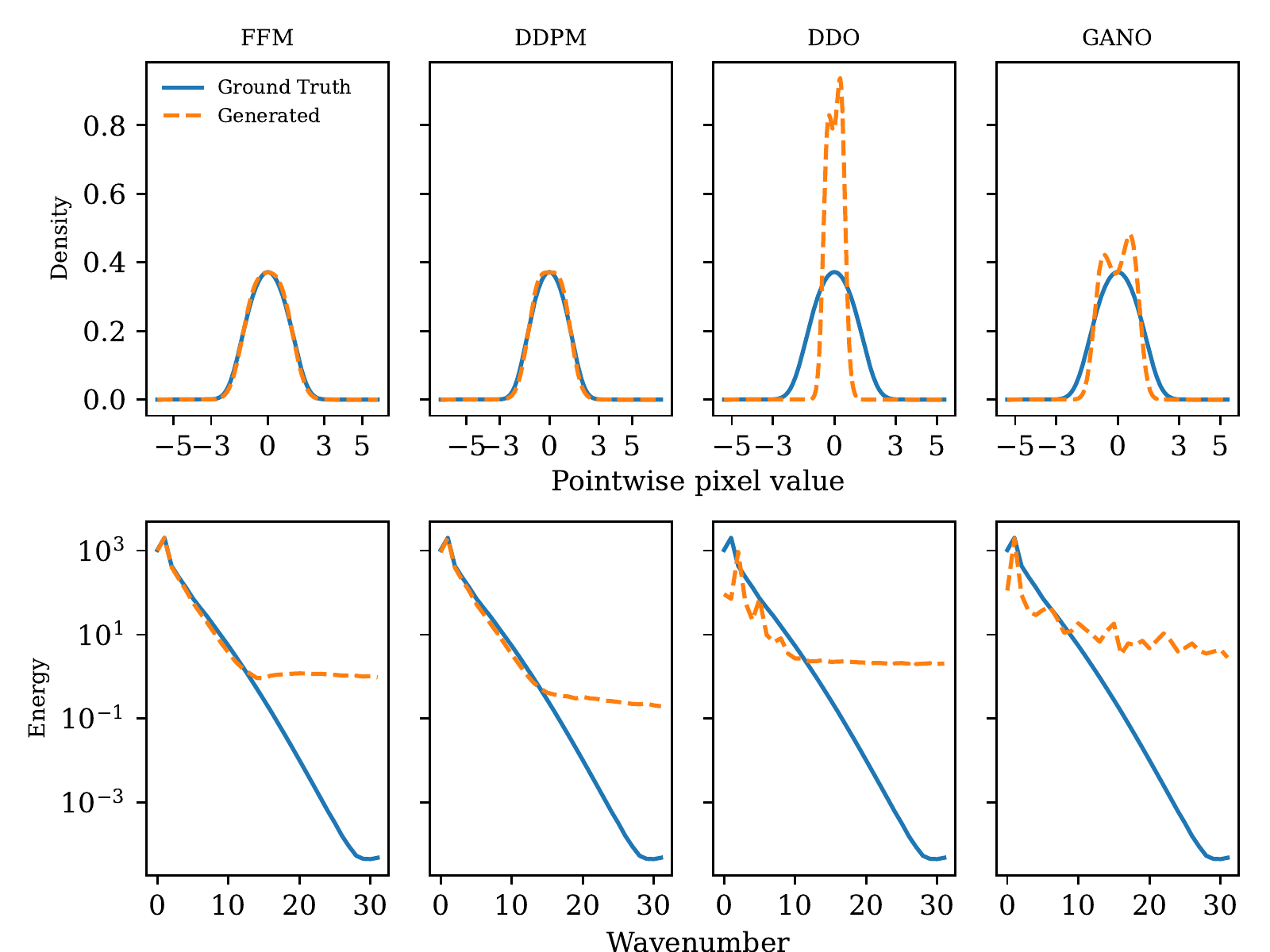}
        \caption{Additional visualizations corresponding to the samples in Figure \ref{fig:ns_samples} and Table \ref{tab:ns_results}. We use 1000 samples from each of the models.}
        \label{fig:ns_density_spectrum}
    \end{figure}

\section{Conditional Models}
\label{appendix:conditional_models}

    In addition to unconditional generation of functions, we demonstrate that our method can be extended to perform conditional generation. That is, we have access to side information $z \in \R^d$ (assumed to be finite dimensional) and we are interested in sampling from the conditional data measure $\nu(\d f \mid z)$. For instance, $z$ could be a collection of observed values of some function, and we may be interesting in generating functions which interpolate (or extrapolate) these given observations. We describe two approaches to performing conditional generation: one based on a modified training process, and a second based on a modified sampling process. In Figure \ref{fig:conditional}, we demonstrate these two approaches on the AEMET dataset.

    \paragraph{Conditional Training.} Using the unconditional paths of measures $\mu_t^f$ as described in the unconditional setting, we may define a conditional marginal $\mu_t(\d f \mid z)$ by mixing over $\d \nu(f \mid z)$, i.e. ${\mu_t(A \mid z) := \int_\FF \mu_t^f(A) \d \nu(f \mid z)}$.

    As long as $\mu_t^f$ is concentrated around $f$, then $\mu_t(\d f \mid z) \approx \nu(\d f \mid z)$. Note that this condition is satisfied for the paths of measures constructed for unconditional generation, and hence no modification is necessary. However, modifying the conditional measures to account for the information $z$ could potentially be beneficial, and we leave exploration of such design choices to future work. In all, we obtain a modified loss function
    
    \begin{equation}
        \JJ_{\text{C}}(\theta) = \E_{t\sim\UU[0,1], z \sim q(z), f \sim \nu(\d f \mid z), g \sim \mu_t^f} \left[ \norm{ v_t^f(g) - u_t(f \mid z, \theta)  }^2\right].
    \end{equation}
    
    In other words, we simply adapt our model architecture to also take in conditioning information $z$ at training time. In practice, because $z$ is assumed to be finite dimensional, we concatenate $z$ to the input of our FNO model \citep{li2020fourier}. We note that a similar loss appears in the context of Flow Matching generative models for video, as proposed by \citep{davtyan2022randomized}.

    \paragraph{Conditional Sampling.} As an alternative, we may instead modify the sampling process to account for $z$. This allows one to train an unconditional model and sample conditionally at generation time (in contrast to the conditional training setup, which only allows you to condition on the particular form of $z$ you have trained on). Here, we assume that $z = (\vec{x}, \vec{y})$ consists of a collection of function observations, and that we would like to generate functions whose values match those observed in $z$.

    In order to achieve this, at time $t \in [0, 1]$, we take a step as dictated by our ODE solver and model vector field to obtain a function $\tilde{f}_t$. Next, we flow the information contained in $z$ forwards for $t$ seconds along the conditional vector field designated by our model to obtain $z_t = (\vec{x}, \vec{y}_t)$. Then, we set $f_t(\vec{x}) = \vec{y}_t$. This approach can be seen as an extension of the ILVR method, which has been successfully applied to diffusion models for conditional image generation \citep{choi2021ilvr} and diffusion models for conditional function generation \citep{kerrigan2022diffusion}. 
\end{document}